\documentclass[10pt,twocolumn,letterpaper]{article}
\usepackage[dvipsnames]{xcolor}
\usepackage{iccv}
\usepackage{times}
\usepackage{epsfig}
\usepackage{graphicx}
\usepackage{amsmath}
\usepackage{amssymb}
\usepackage[ruled,linesnumbered]{algorithm2e}

\SetCommentSty{mycommfont}
\usepackage{microtype}
\usepackage{subfigure}
\usepackage{booktabs} 
\usepackage{amsthm}
\usepackage{thmtools}
\usepackage{microtype}
\usepackage{fixmath}
\usepackage{bbm}
\usepackage{enumitem}
\setitemize{itemsep=0.5pt,topsep=1pt,parsep=1pt,partopsep=1pt}
\setenumerate{itemsep=0.5pt,topsep=1pt,parsep=1pt,partopsep=1pt}
\usepackage{amsxtra}
\usepackage{amsfonts}
\usepackage{multirow}
\usepackage{color}
\usepackage{mathtools}
\usepackage{capt-of}
\usepackage{float}
\usepackage{bm}
\usepackage{pgfplots}
\usepackage{tikz}
\usetikzlibrary{positioning}
\usetikzlibrary{fit}
\usetikzlibrary{shapes.geometric} 
\usetikzlibrary[shapes.geometric] 
\usetikzlibrary{shapes.misc}
\usepackage{scalerel,stackengine}
\usepackage{tikzscale}
\usepackage{tabularx,booktabs}
\newcolumntype{Y}{>{\centering\arraybackslash}X}

\newtheorem{definition}{Definition}

\newtheorem{lemma}{Lemma}
\newtheorem*{lemma*}{Lemma}
\newtheorem{proposition}{Proposition}
\newtheorem*{proposition*}{Proposition}
\newtheorem*{theorem*}{Theorem}
\newtheorem{remark}{Remark}

\usepackage[pagebackref=true,breaklinks=true,letterpaper=true,colorlinks,bookmarks=false]{hyperref}

\makeatletter
\@ifpackageloaded{hyperref}%
  {\newcommand{\mylabel}[2]
    {\protected@write\@auxout{}{\string\newlabel{#1}{{#2}{\thepage}%
      {\@currentlabelname}{\@currentHref}{}}}}}%
  {\newcommand{\mylabel}[2]
    {\protected@write\@auxout{}{\string\newlabel{#1}{{#2}{\thepage}}}}}
\makeatother

\newcommand\algref[1]{%
  \hyperref[#1]{\textcolor{black}{\ref*{#1}}}}

\def\myparagraph#1{\vspace{1pt}\noindent{\bf #1~~}}

\iccvfinalcopy 

\def\iccvPaperID{3559} 

\ificcvfinal\fi

\DeclareMathOperator*{\argmin}{arg\,min} 
 
\DeclareMathOperator*{\la}{\langle} 
\DeclareMathOperator*{\ra}{\rangle} 
\def\conv{\operatorname{conv}}

\providecommand{\abs}[1]{\lvert#1\rvert}
\providecommand{\norm}[1]{\lVert#1\rVert}
\newcommand{\bx}{\bm{x}}
\newcommand{\by}{\bm{y}}
\newcommand{\btheta}{\bm{\theta}}

\newcommand{\bphi}{\bm{\phi}}
\newcommand{\bTheta}{\bm{\Theta}}
\newcommand{\bPhi}{\bm{\Phi}}
\newcommand{\bsigma}{\bm{\sigma}}
\newcommand{\bomega}{\bm{\omega}}
\newcommand{\etab}{\bm{\eta}}
\newcommand{\blambda}{\bm{\lambda}}
\newcommand{\R}{\mathbb{R}}

\newcommand{\SX}{\mathcal{X}}
\newcommand{\SY}{\mathcal{Y}}
\newcommand{\SG}{\mathcal{G}}
\newcommand{\SV}{\mathcal{V}}
\newcommand{\SE}{\mathcal{E}}

\newcommand{\FG}{\mathbb{G}}
\newcommand{\FD}{\mathbb{D}}
\newcommand{\FV}{\mathbb{V}}
\newcommand{\FE}{\mathbb{E}}
\newcommand{\FA}{\mathbb{A}}

\definecolor{mycolor1}{RGB}{
	103,169,207
}
\definecolor{mycolor2}{RGB}{
	239,138,98
}

\colorlet{bfactorcolor}{mycolor1}
\colorlet{bpotentialcolor}{mycolor2}

\makeatletter
\newcommand{\nosemic}{\renewcommand{\@endalgocfline}{\relax}}
\newcommand{\dosemic}{\renewcommand{\@endalgocfline}{\algocf@endline}}
\newcommand{\pushline}{\Indp}
\let\oldnl\nl
\newcommand{\nonl}{\renewcommand{\nl}{\let\nl\oldnl}}
\makeatother

\begin{document}
	
	\title{Bottleneck potentials in Markov Random Fields}
	
    \author{Ahmed Abbas $^{1,2}$ \qquad Paul Swoboda$^1$ \\
       \small $^1$Max Planck Institute for Informatics, \, $^2$ZEISS%
        }
	
	\maketitle
    \ificcvfinal\thispagestyle{empty}\fi
	\maketitle
	\begin{abstract}
We consider general discrete Markov Random Fields~(MRFs) with additional bottleneck potentials which penalize the maximum (instead of the sum) over local potential value taken by the MRF-assignment.
Bottleneck potentials or analogous constructions have been considered in
(i) combinatorial optimization (e.g.\ bottleneck shortest path problem, the minimum bottleneck spanning tree problem, bottleneck function minimization in greedoids),
(ii) inverse problems with $L_{\infty}$-norm regularization, and
(iii) valued constraint satisfaction on the $(\min,\max)$-pre-semirings.
Bottleneck potentials for general discrete MRFs are a natural generalization of the above direction of modeling work to Maximum-A-Posteriori (MAP) inference in MRFs.
To this end, we propose MRFs whose objective consists of two parts:
terms that factorize according to (i) $(\min,+)$, i.e.\ potentials as in plain MRFs, and (ii) $(\min,\max)$, i.e.\ bottleneck potentials.
To solve the ensuing inference problem, we propose high-quality relaxations and efficient algorithms for solving them.
We empirically show efficacy of our approach on large scale seismic horizon tracking problems.
\end{abstract}
	\section{Introduction}
In the field of computer vision MRFs have found many applications such as image segmentation, denoising, optical flow, 3D reconstruction and many more, see~\cite{opengm_benchmark_IJCV} for a non-exhaustive overview of problems and algorithms.
The above application scenarios are modelled well as a sum of unary/pairwise/ternary/$\ldots$ potentials on the underlying graphical structure.
More generally, whenever the error of fitting a model to data is captured by local terms, the objective can factorize into a sum of local error terms.
However, such objectives are not always appropriate.
In some computer vision applications, a single error can entail subsequent errors, rendering the solution useless and local terms are unable to properly penalize this.
Prominent examples are tracking problems~\cite{MST,2d_seismic_horizon_tracking_shortest_path,rathke2014probabilistic, BSP_Mosaicking}, where making a single error and following a wrong track results in low accuracy nonetheless.
In inverse problems, $L_{\infty}$-norm regularization penalizes the maximum deviation from the fitted model and is appropriate e.g. for some types of group sparsity~\cite{l1_infinity_icml_2009,liu2009blockwise,network_flow_for_structured_sparsity}.
In all of the above scenarios global potentials, which penalize the maximum value assignment w.r.t.\ a given set of local costs are the appropriate choice.
We call this maximum value a \textit{bottleneck} and the aim is to find a configuration such that its bottleneck has minimum cost.
Formally, optimizing a bottleneck objective can be written as
\begin{equation}
\label{eq:general-bottleneck-problem}
\min_{x \in X} \left( \max_{i} \{\psi_i \cdot x_i\} \right) 
\end{equation}
where $X \subset \R^n$ is the space of feasible elements of the optimization problem and $\psi \in \R^n$ is a real valued vector.
Additionally, the bottleneck objective can be written as the infinity norm $\norm{\psi \odot x}_{\infty}$, where $\odot$ is the Hadamard product.

\section{Related work}
Bottleneck-type objectives occur throughout many subfields of mathematical programming. The optimization problems are also often called min-max problems. 

\myparagraph{Bottleneck potentials in MRFs.}
The case of pairwise binary MRFs with bottleneck potentials has been addressed in~\cite{InfinityNormSegmentation} and has been applied to image segmentation.
As noted by the authors of~\cite{InfinityNormSegmentation}, incorporating a bottleneck term gives better segment boundaries and resolves \lq small cuts\rq\ (or shrinking bias) of graph cut~\cite{Boykov2006}.
In~\cite{PowerWaterShed-1} the authors interpret $L_{p}$-norm regularization for discrete labeling problems and for $p\in [1,\infty)$ as MAP-inference in MRFs and propose approximating the bottleneck potential corresponding to $L_{\infty}$ via a high value of $p$.
In~\cite{KohliCVPR2010} a more general class of MAP-inference with higher order costs were proposed which include the bottleneck potentials we investigate.
In~\cite{Pansari_2017_CVPR} the authors propose a model that can be specialized to bottleneck costs with linear distances and give a primal heuristic for solving it.


Furthermore, labeling problems containing only bottleneck objectives were considered in ~\cite{min-max-message-passing, min-max-propagation, consistent-labeling, consistent-labeling-2}, where~\cite{consistent-labeling, consistent-labeling-2} devised algorithms for special cases of the pure min-max labeling problem, and ~\cite{min-max-message-passing, min-max-propagation} devised message passing schemes with application in parallel machine scheduling.

In contrast to previous works~\cite{PowerWaterShed-1,InfinityNormSegmentation,min-max-message-passing, min-max-propagation, consistent-labeling, consistent-labeling-2} we investigate the mixed problem where ordinary and bottleneck potentials are both present in the same inference problem.
Unlike~\cite{InfinityNormSegmentation} we allow arbitrary MRFs.
Also unlike the algorithm in~\cite{PowerWaterShed-1,Pansari_2017_CVPR}, our algorithms are based on a linear programming relaxation resulting in a more rigorous approach. 
Lastly, in comparison to~\cite{KohliCVPR2010} we propose a stronger relaxation (theirs corresponds to~\ref{lemma:local-polytope-noninteger}) and an algorithm that scales to large problem sizes.

\myparagraph{Combinatorial bottleneck optimization problems.}
Many classical optimization problems have bottleneck counterparts, e.g.\ bottleneck shortest paths~\cite{BSP}, bottleneck spanning trees~\cite{bottleneckMST} and bottleneck function optimization in greedoids~\cite{korte1991greedoids}.
The above works consider the case of bottleneck potential only.
More similar to our work is the mixed case of linear and bottleneck costs, as investigated for the shortest and bottleneck path problem in~\cite{BSPSP}.

\myparagraph{$L_{\infty}$-norm regularization in inverse problems.}
For inverse problems $\min_{x \in \R^n} \norm{Ax - b}_p$ with uniform noise, the appropriate choice of norm is $p=\infty$~\cite{ConvexOptimizationBoyd}[Chapter~7.1.1].
Extensions of this basic $L_{\infty}$-norm are used in~\cite{l1_infinity_icml_2009} for multi-task learning and in~\cite{liu2009blockwise} for the multi-task lasso.
More generally, mixed norms with $L_{\infty}$-norms are useful for problems where group sparsity is desirable.
Proximal~\cite{l1_infinity_icml_2009}, block coordinate descent~\cite{liu2009blockwise} as well as network flow techniques~\cite{network_flow_for_structured_sparsity} have been proposed for numerical optimization of such problems.

\myparagraph{Semiring-based constraint satisfaction.}
In semiring-based constraint satisfaction problems the goal is to compute $\bigoplus_{x_V \in X} \bigodot_{A \in E} f_A(x_A)$, where $X$ is a space of labellings on a node set $V$, $E$ is collection of subsets of $V$, and $f_A$ are functions that depend only on nodes in the subset $A$~\cite{Werner-Soft08,Werner-CVWW-2007,semiring_induced_valuation_algebras_Kohlas_Wilson,Bistarelli1999,bistarelli2004semirings}.
Popular choices for the pair $(\bigoplus,\bigodot)$ is the $(\min,+)$-semiring which corresponds to MAP-inference and the $(+,\times)$-semiring which corresponds to computing the partition function.
In contrast to the above semirings, the algebra corresponding to bottleneck potentials $(\min,\max)$ is only a pre-semiring (the distributive law does not hold) and hence classical arc consistency algorithms as discussed in~\cite{Werner-Soft08,Werner-CVWW-2007} are not applicable and specialized methods are needed. 
For an in-depth study of the $(\min,\max)$-pre-semiring we refer to~\cite{cuninghame1979minimax}.
Our case however does not completely fit into (pre-)semiring based constraint satisfaction setting, since we are concerned with the mixed case in which both the $(\min,+)$-semiring and the $(\min,\max)$-pre-semiring occur together in a single optimization problem.

\myparagraph{Applications in horizon tracking.}
	The seismic horizon tracking problem, i.e. identifying borders between layers of various types of rock beds, has been addressed in ~\cite{MST,2d_seismic_horizon_tracking_shortest_path,PDE,3d_seismic_horizon_tracking_path_constraints} from the computational perspective.
    The authors in~\cite{MST} use a greedy method inspired by the minimum spanning tree problem to that end.
    In~\cite{2d_seismic_horizon_tracking_shortest_path} the authors propose to solve a shortest path problem to track seismic horizons along a 2-D sections of the original 3-D volume.
    In~\cite{PDE,3d_seismic_horizon_tracking_path_constraints} the authors set up linear equations to solve the horizon tracking problem.
    
    Most similar to our work is the minimum spanning tree inspired approach of~\cite{MST}. 
    In contrast to~\cite{MST}, we consider a rigorously defined optimization problem for which we develop a principled LP-based approach instead of greedily selecting solutions.
   Conceptually, the 2-D shortest path method~\cite{2d_seismic_horizon_tracking_shortest_path} is also similar to ours, however it cannot be extended to the 3-D setting we are interested in. Moreover, we allow for a more sophisticated objective function than~\cite{2d_seismic_horizon_tracking_shortest_path}.
Methods of~\cite{PDE,3d_seismic_horizon_tracking_path_constraints} do not use optimization at all but solve linear systems and require more user intervention.



\myparagraph{Contribution \& organization.}
Section~\ref{sec:model} introduces bottleneck potentials and their non-linear generalizations for general discrete MRFs. We also consider the mixed problem of MAP-inference for a combination of ordinary MRF-costs w.r.t.\ the $(\min,+)$-semiring and bottleneck potentials.
In Section~\ref{sec:algorithms} we derive special algorithms to solve the problem for chain graphs via a dynamic shortest path method, and for graphs without pairwise interactions via an efficient enumerative procedure.
Combining these two special cases, we derive a high-quality relaxation for the case of general graphs.
To solve this relaxation, we propose an efficient dual decomposition algorithm.
In Section~\ref{sec:experiments} we show empirically that our approach results in a scalable algorithm that gives improved accuracy on seismic horizon tracking as compared to MAP-inference in ordinary MRFs and a state-of-the-art heuristic.

Code and datasets are available at \url{https://github.com/LPMP/LPMP}.


\section{MRFs with bottleneck potentials}
\label{sec:model}
First, we will review the classical problem of Maximum-A-Posteriori (MAP) inference in Markov Random Fields (MRF).
Second, we will introduce the bottleneck labeling problem which extends MAP-MRF by additional bottleneck term that penalize the maximum value of potentials taken in an assignment (as opposed to the sum for MRFs).

\subsection{Markov Random Fields:}
A graph will be a tuple $G=(V,E)$ with undirected edges $E \subset \begin{pmatrix} V \\ 2 \end{pmatrix}$.
To each node $i \in V$ a label set $\SX_i$ is associated.
To each node $i \in V$ a \emph{unary potential} $\theta_i: \SX_i \rightarrow \R$ is associated, and to each edge $ij \in E$ a \emph{pairwise potential} $\theta_{ij}: \SX_i \times \SX_j \rightarrow \R$.
We will call $\SX_V = \prod_{i \in V} \SX_i$ the \emph{label space} and $x \in \SX$ a \emph{labeling}.
For subsets $U \subset V$ we define $\SX_U = \prod_{i \in U} \SX_i$, analogously we refer to labels $x_U \in \SX_U$.
In particular, $x_i$ refers to a node labeling and $x_{ij} = (x_i,x_j)$ to an edge labeling.
A tuple $(G,\SX,\btheta)$ consisting of a graph, corresponding label space and potentials is called a Markov Random Field (MRFs).
The problem
\begin{equation}
\label{eq:MRF}
\min_{\bx \in \SX} \btheta(\bx), \quad \btheta(\bx) := \sum_{i \in V} \theta_i(x_i) + \sum_{ij \in E} \theta_{ij}(x_{ij})\,
\end{equation}
is called the Maximum-A-Posteriori (MAP) inference problem in MRFs.

\myparagraph{Local Polytope Relaxation:}
We use the over-complete representation to obtain a linear programming relaxation of the optimization problem~\eqref{eq:MRF}.
For $i\in V$ we associate the $k$-th label from $\SX_i$ with one-hot encoding $e_k = (0,\ldots,0,\underbrace{1}_k,0,\dots 0)$ which is a unit vector of length $\abs{\SX_i}$ with a $1$ at the $k$-th location.
In other words, we can write $\SX_i = \{e_1, e_2, \ldots e_{\abs{\SX_i}}\}$ $\forall i \in V$.
Analogously, we can write $\SX_{ij} = \{e_1,\ldots, e_{\abs{\SX_{ij}}}\}$ $\forall ij \in E$.
To obtain a convex relaxation, we define \emph{unary marginals} as $\mu_i \in \conv{\SX_i}$, $i \in V$, and \emph{pairwise marginals} as $\mu_{ij} \in \conv{\SX_{ij}}$, $ij \in E$.
We couple unary and pairwise marginals together to obtain the \emph{local polytope relaxation}~\cite{WernerLocalPolytope}.
\begin{equation}
\Lambda = \left\{\mu \, \middle|\begin{array}{c} 
	\mu_i(x_i) = \sum\limits_{x_j \in \SX_j} \mu_{ij}(x_{ij}), \; \forall ij \in E, \; x_i \in \SX_i \\
	\mu_f \in \conv{\SX_f}, \; \forall f \in V \cup E
\end{array} \right\}
\label{eq:local-polytope-constraints}
\end{equation}
With the local polytope we can relax the problem of MAP-inference in MRFs~\eqref{eq:MRF} as:
 \begin{equation}
 \min_{\mu \in \Lambda} \sum_{i \in V} \la \theta_i, \mu_i \ra + \sum_{ij \in E} \la \theta_{ij}, \mu_{ij} \ra\,
 \label{eq:local-polytope}
 \end{equation}
Note that the relaxation~\eqref{eq:local-polytope} subject to constraint~\eqref{eq:local-polytope-constraints} is tight for some graphs such as trees and for special families of cost functions including different forms of submodularity~\cite{KolmogrovCSPs}.

\subsection{Bottleneck labeling problem:}
Given an MRF, we associate to it a second set of potentials, which we call bottleneck potentials.
As opposed to MRF, however, the corresponding assignment cost is not given by the sum of individual potential values but by their maximum.
The goal of inference in a pure bottleneck labeling problems (i.e.\ with all zero MRF potentials) is thus to find a labeling such that the maximum bottleneck potential value taken by the labeling is minimal.

\begin{definition}[Bottleneck labeling problem]
	Let an MRF be given. Additionally, let \emph{unary bottleneck potentials} $\phi_i: \SX_i \rightarrow \R$ $\forall i \in V$ and \emph{pairwise bottleneck potentials} $\phi_{ij}: \SX_i \times \SX_j \rightarrow \R$ $\forall ij \in E$  are also given.
    We call the set of all possible values taken by bottleneck potentials as \emph{bottleneck values}, i.e.\ 
    \begin{equation}
    B = \{ \phi_f(x_f) : f \in V \cup E, x_f \in \SX_f\}\,.
    \end{equation}
    Let \emph{bottleneck costs} $\zeta: B \rightarrow \R$ be given.
	The bottleneck labeling problem is defined as
	\begin{subequations}\label{eq:bottleneck-labeling-problem}
		\begin{align}
		\min\limits_{\bx \in \SX, b \in B}  & \btheta(\bx) + \zeta(b)
		\label{eq:bottleneck-labeling-objective} \\
		\text{s.t.} \qquad & \phi_f(x_f)  \leq b \quad \forall f \in V \cup E
		\label{eq:bottleneck-constraints}
		\end{align}
	\end{subequations}
and $\btheta(\bx)$ is defined in \eqref{eq:MRF}.
\end{definition}
It is straightforward to adapt our work to the case of multiple bottlenecks and triplet/quadrupelt/$\ldots$ potentials.
However we focus on models with only a single bottleneck and pairwise potentials for simplicity.

A special case of the problem \eqref{eq:bottleneck-labeling-problem} was considered by~\cite{InfinityNormSegmentation} where the authors only allow binary labels, special form of MRF costs and $\zeta(b) = b$.
Additionally, heuristics for extending to the multi-label case were given in ~\cite{PowerWaterShed-1}.
Below we propose exact algorithms for multi-label chain graphs and an LP-relaxation for general graphs.

\myparagraph{Example.}
In the special case when $\zeta(b) = b$ and MRF costs are zero i.e. $\btheta = 0$ then the problem reduces to a pure bottleneck labeling problem as:
\begin{equation}
\label{eq:bottleneck-labeling-problem-pure}
\min\limits_{\bx \in \SX}  \max\limits_{f \in V \cup E}(\phi_f(x_f)) \\
\end{equation}

Note that if we are know the optimal bottleneck value $b^*$  in $B$, then the bottleneck labeling problem can be reduced to the MAP-inference problem in MRFs.
This reduction can be done by setting the unary and pairwise MRF costs to $\infty$ for the labelings which have bottleneck potentials greater than the optimal value $b^*$. i.e.
\begin{equation}
		\theta_e(x_e) \coloneqq \infty \quad \forall e, x_e: \phi_e(x_e) > b^* \\
\label{eq:bottleneck-mrf-reduction}
\end{equation}
 Then, the constraints ~\eqref{eq:bottleneck-constraints} will be automatically satisfied by a feasible solution of the MAP-inference problem.


\section{Algorithms:}
\label{sec:algorithms}
In this section, algorithms for solving the bottleneck labeling problem are proposed. 
We first present efficient and exact algorithms for edge-free graphs and chain graphs.
Later, we will use these two algorithms for solving the problem on general graphs using dual-decomposition. 
The algorithms for edge-free graphs and chains are designed with the end goal of using them inside dual decomposition for general graphs and therefore contain extra steps.

\subsection{Bottleneck labeling with unary potentials:}
Assume that the graphical model does not contain any edges i.e.\ $E = \varnothing$, so only unary potentials need to be considered. 
The problem~\eqref{eq:bottleneck-labeling-problem} in this case can be efficiently solved with Algorithm~\algref{alg:unary-bottleneck-labeling-problem}.\\ 
\begin{algorithm}[!ht]
\SetAlgorithmName{unary\_bottleneck}{unary\_bottleneck}{List of algorithms}
\caption{bottleneck labeling on unary graph}
	\mylabel{alg:unary-bottleneck-labeling-problem}{\textbf{unary\_bottleneck}}
	\KwData{ \\
		MRF without edges: $(G=(V,\varnothing), \SX, \{\theta_i\}_{i \in V})$, \\
		Bottleneck potentials: $\{\phi_i\}_{i \in V}$.
	}
	\KwResult{ 
		Costs of labelings: $M = \left\{ (b,c) : 
		\begin{array}{c}
		c = \min_{\bx} \sum_{i \in V} \theta_i(x_i), \\
		\phi_i(x_i) \leq b \quad \forall i \in V
		\end{array}
		\right\}
		$.}
	\tcp{Merge labels of all nodes in $\Theta$:}
	$\Theta = \{(i,x_i) : i \in V, x_i \in \SX_i \}$\; 
	Sort $\Theta$ according to $\phi_i(x_i)$ \label{alg:unary-bottleneck-labeling-problems-sorting}\;
	$c = 0$, $l = \infty \in \R^{|V|}$, $S = \varnothing$, $M = \varnothing$\;
	\For{$(i, x_i) \in \Theta$ in ascending order} {
		\If{$i \notin S$} {
			$ S = S \cup \{i\}$\; 
			$ c = c + \theta_i(x_i)$\;
			$l_i = \theta_i(x_i)$\;
		}
		\If{$\theta_i(x_i) < l_i$} {
			$c = c - l_i + \theta_i(x_i)$\;
			$l_i = \theta_i(x_i)$\; 
		}
		\If{$S = V$} {
			$ M = M \cup \{(\phi_i(x_i), c)\}$\;
		}
	}
\end{algorithm}
Algorithm \algref{alg:unary-bottleneck-labeling-problem} enumerates all bottleneck values in ascending order. For every bottleneck value $b$, an optimal node labeling can be found by choosing for each node $i$ the best label w.r.t.\ potentials $\theta_i$ that is feasible to bottleneck constraints~\eqref{eq:bottleneck-constraints}.
The costs of such labelings are stored in $M$ .
Updating the best node labeling between consecutive bottleneck values can be done by only checking the nodes for which~\eqref{eq:bottleneck-constraints} has changed the feasible set.
Finally, the optimum bottleneck value can be computed by
\begin{equation}
	(b^*, c^*) = \argmin_{(b,c) \in M} \left[ c + \zeta(b) \right]
	\label{eq:optimal-b-c}
\end{equation}
Once the optimal bottleneck value $b^*$ is computed, the problem can be reduced to inference in MRF by disallowing the configurations which have bottleneck potentials greater than $b^*$ as mentioned in \eqref{eq:bottleneck-mrf-reduction}.

\begin{proposition}
\label{prop:unary-bottleneck-labeling-problem-runtime}
The runtime of Algorithm~\algref{alg:unary-bottleneck-labeling-problem} is $\mathcal{O}(L\log{}L + L)$ where $L = \sum_{i \in V} |\SX_i|$
\end{proposition}
The most expensive operation in Algorithm~\algref{alg:unary-bottleneck-labeling-problem} is sorting. However, the sorting can be reused when the algorithm is run multiple times with varying linear potentials $\theta$, which is the case in our dual decomposition approach for general graphs.

\subsection{Bottleneck labeling problem on chains:}   
Assume $V = [n]$ and $E = \{(1,2),(2,3),\ldots,(n-1,n)\}$ is a chain. Even though the relaxation over the local polytope for inference in chains is tight for pairwise MRFs, introducing the bottleneck potential (as also done in~\cite{KohliCVPR2010}) destroys this property which is demonstrated in the Appendix in Lemma~\ref{lemma:local-polytope-noninteger}.  Therefore, we propose an efficient algorithm to solve the bottleneck labeling problem~\eqref{eq:bottleneck-labeling-problem} exactly on chains. \\
First, we note that the MAP-MRF  problem~\eqref{eq:MRF} can be modeled through a shortest path problem in a directed acyclic graph. Figure~\ref{fig:shortest-path-illustration} illustrates the construction for the case of $n=3$. 
We treat each label $x_i$ of a variable $i$ in $V$ in the chain as a pair of nodes $x_i$, and $\overline{x}_i$ in the shortest path digraph $D=(W,A)$.
Each unary cost $\theta_i(x_i)$ becomes an arc cost for $(x_i,\overline{x}_i)$.
Each pairwise cost $\theta_{ij}(x_i,x_j)$ becomes an arc cost $(\overline{x}_i, x_j)$.
The source and sink nodes $s, t$ are connected with the labels of first and last node of the graphical model resp. with zero costs for modeling the shortest path problem.
Algorithm~\algref{alg:chain-digraph-conversion} in the Appendix gives the general construction. \\
The shortest $s,t$-path $P$ in graph $D$ having cost $\bsigma(P)$ corresponds to an optimum labeling for $\btheta(\bx)$ from~\eqref{eq:bottleneck-labeling-problem} in the chain graphical model $G$.
\begin{figure}[H]
	\centering
	\resizebox{.45\textwidth}{!}{\begin{tikzpicture}
	\node[circle,fill] (x11) {}; 
	\node[circle,fill,below of=x11] (x12) {};
	\node[circle,fill,below of=x12] (x13) {};
	\node[circle,draw=black,fill=white,right of=x11] (x11') {};
	\node[circle,draw=black,fill=white,right of=x12] (x12') {};
	\node[circle,draw=black,fill=white,right of=x13] (x13') {};   
	\node[label=above:$u$,fit=(x11) (x13')] {};
	\node[draw,inner sep=2mm,label=below:$\theta_u$,fit=(x11) (x13')] {};
	
	\node[circle,fill,right=3cm of x11] (x21) {}; 
	\node[circle,fill,below of=x21] (x22) {}; 
	\node[circle,fill,below of=x22] (x23) {}; 
	\node[circle,draw=black,fill=white,right of=x21] (x21') {};
	\node[circle,draw=black,fill=white,right of=x22] (x22') {};
	\node[circle,draw=black,fill=white,right of=x23] (x23') {};   
	
	\node[label=above:$v$,fit=(x21) (x23')] {};
	\node[draw,inner sep=2mm,label=below:$\theta_v $,fit=(x21) (x23')] {};
	
	\node[circle,fill,right=3cm of x21] (x31) {}; 
	\node[circle,fill,below of=x31] (x32) {}; 
	\node[circle,fill,below of=x32] (x33) {};
	\node[circle,draw=black,fill=white,right of=x31] (x31') {};
	\node[circle,draw=black,fill=white,right of=x32] (x32') {};
	\node[circle,draw=black,fill=white,right of=x33] (x33') {};   
	
	\node[label=above:$w$,fit=(x31) (x33')] {};
	\node[draw,inner sep=2mm,label=below:$\theta_w $,fit=(x31) (x33')] {};
	
	\node[circle,fill=gray,left=3cm of x12] (s) {};
	\node[label=above:$s$,fit=(s)] {};
	
	\node[circle,fill=gray,right=3cm of x32] (t) {};
	\node[label=above:$t$,fit=(t)] {};
	
	\draw[->] (s) to (x11);
	\draw[->] (s) to (x12);
	\draw[->] (s) to (x13);
	
	\draw[->] (x11) to (x11');
	\draw[->] (x12) to (x12');
	\draw[->] (x13) to (x13');
	
	\draw[->] (x11') to (x21);
	\draw[->] (x11') to (x22);
	\draw[->] (x11') to (x23);
	\draw[->] (x12') to (x21);
	\draw[->] (x12') to (x22);
	\draw[->] (x12') to (x23);
	\draw[->] (x13') to (x21);
	\draw[->] (x13') to (x22);
	\draw[->] (x13') to (x23);
	\node[label=below:$\theta_{uv}$,fit=(x11') (x23)] {};
	
	\draw[->] (x21) to (x21');
	\draw[->] (x22) to (x22');
	\draw[->] (x23) to (x23');
	
	\draw[->] (x21') to (x31);
	\draw[->] (x21') to (x32);
	\draw[->] (x21') to (x33);
	\draw[->] (x22') to (x31);
	\draw[->] (x22') to (x32);
	\draw[->] (x22') to (x33);
	\draw[->] (x23') to (x31);
	\draw[->] (x23') to (x32);
	\draw[->] (x23') to (x33);
	\node[label=below:$\theta_{vw}$,fit=(x21') (x33)] {};
	
	\draw[->] (x31) to (x31');
	\draw[->] (x32) to (x32');
	\draw[->] (x33) to (x33');
	
	\draw[->] (x31') to (t);
	\draw[->] (x32') to (t);
	\draw[->] (x33') to (t);
\end{tikzpicture}\unskip}
	\caption[]{Shortest path network for a chain MRF with nodes $\{u,v,w\}$ and 3 labels for each node. The state $x_i$ of each node $i \in V$ represented by \tikz \node [circle, fill=black, scale = 0.7]{}; has been duplicated to $\overline{x}_i$ represented by \tikz \node [circle, draw=black, fill=white, scale = 0.7]{};  
		to introduce unary potentials as arc costs. Text below the arcs represent their costs. Missing arcs have cost $\infty$.}
	\label{fig:shortest-path-illustration}
\end{figure}
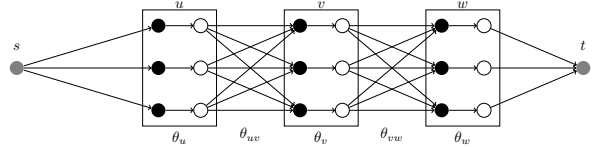
Algorithm~\algref{alg:chain-bottleneck-labeling-problem} uses the constructed directed acyclic graph to solve the bottleneck labeling problem~\eqref{eq:bottleneck-labeling-problem} on chain MRF by an iterative procedure.
For each bottleneck thresholds $b$ in $B$, the shortest path is found such that only those edges are used whose bottleneck potentials are within the threshold. This way a matrix $M$ of all possible pairs of bottleneck values $b$ and shortest path costs $c$ is computed. The optimal bottleneck value $b^*$ can be then found by~\eqref{eq:optimal-b-c}. \\
Naively implementing Algorithm~\algref{alg:chain-bottleneck-labeling-problem} by computing a shortest path from scratch in each iteration of the loop will result in high complexity. Note that when iterating over bottleneck threshold values in ascending order, exactly one edge is added per iteration.
Therefore, dynamic shortest path algorithms can be used for recomputing shortest paths more efficiently. Additionally, since the graph $D$ is directed and acyclic, the shortest path in each iteration can be found in linear time by breadth-first traversal~\cite{shortestPathBook}.
These improvements are detailed in Algorithm~\algref{alg:dynamic-shortest-path-for-chains} in the Appendix.
\begin{algorithm}
\SetAlgorithmName{chain\_bottleneck}{chain\_bottleneck}{List of algorithms}
	\caption{bottleneck labeling on chain graphs}
	\KwData{
	    Chain MRF: $R = ((V,E), \SX, \{\theta_f\}_{f \in V \cup E})$, \\
		Bottleneck potentials:  $\{\phi_f(x_f)\} _ {f \in V \cup E}$, \\
	}
	\KwResult{ \\
		$M = \left\{ (b,c) : 
		\begin{array}{c}
		c = \min_{\bx} \sum_{f \in F} \theta_f(x_f), \\
		\phi_f(x_f) \leq b \quad \forall f \in V \cup E
		\end{array}
		\right\}$} 
	\tcp{Represent chain as DAG:}
	$(D=(W,A), \bsigma, \bomega, s, t) \leftarrow ~\algref{alg:chain-digraph-conversion}(R, \bphi)$ \;
	\tcp{Shortest path distances:}
	$d(s) \coloneqq 0$, \hfill $d(w) \coloneqq \infty, \quad \forall w \in W\setminus\{s\}$\;
	Sort $A$ according to values $\bomega$ \label{alg:chain-bottleneck-labeling-problem-sorting}\;
	\For{$(n_i, n_j) \in A$ in ascending order \label{alg:chain-bottleneck-labeling-problem-arc-addition-begin}} {
		$A' = (n_i, n_j) \cup A' $ \;
		\tcp{Update $d$ w.r.t arc-costs $\bsigma$:}
		$d = \algref{alg:dynamic-shortest-path-for-chains}((W,A'),d,\bsigma,(n_i, n_j))$ \; \label{alg:chain-bottleneck-labeling-problem-shortest-path-update}
		\tcp{Check for s-t path:}
		\If{$d(t) < \infty$} {
			
			\tcp{Store path costs:}
			$M = M \cup \{(\omega(n_i, n_j), d(t))\}$\;
		}
	}
	\label{alg:chain-bottleneck-labeling-problem-arc-addition-end}
	\mylabel{alg:chain-bottleneck-labeling-problem}{\textbf{chain\_bottleneck}}
\end{algorithm}	
\begin{proposition}
	\label{prop:chain-bottleneck-labeling-problem-runtime}
	The worst case run-time of Algorithm~\algref{alg:chain-bottleneck-labeling-problem} is $\mathcal{O}(|A|^2)$ where, $A$ is the arc-set in underlying graph $D=(W,A)$.
\end{proposition}
While the worst-case runtime of~\algref{alg:chain-bottleneck-labeling-problem} is quadratic in the number of edges of the underlying shortest path graph, the average case runtime is better, since not for every edge a new shortest path needs to be computed and often a shortest path computation can reuse previous results for speedup. Moreover, the sorting operation on line~\algref{alg:chain-bottleneck-labeling-problem-sorting} can be re-used similar to Algorithm~\algref{alg:unary-bottleneck-labeling-problem}.

\begin{remark}
If we have a pure bottleneck labeling problem without the MRF potentials ~\eqref{eq:bottleneck-labeling-problem-pure} then the problem can be solved in linear-time by sorting the bottleneck weights and halving the number of edges by the median edge as mentioned in ~\cite{BSP}. The problem we consider is more general and needs to account for the MRF potentials.
\end{remark}

\begin{remark}
The chain bottleneck labeling problem can also be solved for linear bottleneck cost $\zeta(x) = x$ by the method from~\cite{BSPSP}.
However, the method from~\cite{BSPSP} is not polynomial w.r.t.\ bottleneck values $B$ but requires them to be small integers to be efficient, making it unsuitable for our purposes.
\end{remark}

\begin{remark}[Bottleneck labeling on trees]
The dynamic shortest path algorithm can be modified to work on trees as well, where we replace shortest path computations by belief propagation.
For clarity of presentation we have chosen to restrict ourselves to chains.
\end{remark}

\subsection{Relaxation for the bottleneck labeling problem on arbitrary graphs:}
\label{subsection:arbitrary-graphs}
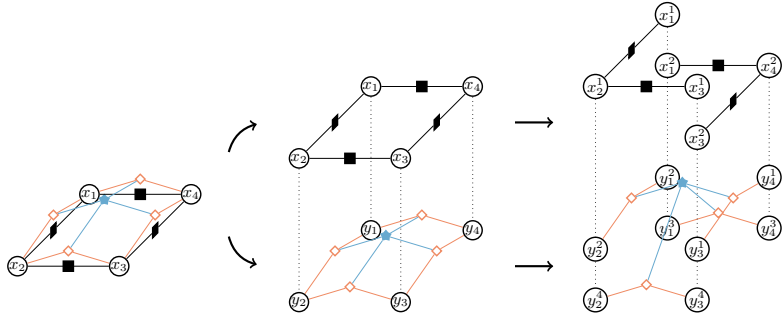
\begin{figure*}[h]
	\centering	
	\resizebox{0.6\textwidth}{!}{
 	\begin{tikzpicture}
[node/.style={circle,draw=black,fill=white,thick,
	inner sep=0.1pt,minimum size=2mm},
edgefactor/.style={rectangle,draw=black,fill=black,thick,
	inner sep=0.1pt,outer sep = 0pt,minimum size=2mm},
edgefactorSheared/.style={trapezium,trapezium left angle=135,trapezium right angle=45,draw=black,fill=black,thick,
	inner sep=0.1pt,minimum size=1mm, minimum width=1mm},
maxEdgefactor/.style={diamond, draw=bpotentialcolor,fill=white,thick,
	inner sep=0.1pt,minimum size=2mm},
maxfactor/.style={star, draw=bfactorcolor,fill=bfactorcolor,thick,
	inner sep=0.1pt,minimum size=2mm}]

\node[node] (left-bot) at (0,0) {\small \small $x_2$};
\node[edgefactor] (bot-factor) [right of=left-bot] {};
\node[] (fake-left-factor) [above right of=left-bot] {};
\node[edgefactorSheared, rotate around = {90:(fake-left-factor.center)}] (left-factor) [above right of=left-bot] {};
\node[node] (right-bot) [right of=bot-factor] {\small $x_3$};
\node[node] (left-top) [above right of=left-factor] {\small $x_1$};
\node[edgefactor] (top-factor) [right of=left-top] {};
\node[node] (right-top) [right of=top-factor] {\small $x_4$};
\node[] (fake-right-factor) [above right of=right-bot] {};
\node[edgefactorSheared, rotate around = {90:(fake-right-factor.center)}] (right-factor) [above right of=right-bot] {};

\node[maxEdgefactor] [yshift = -7mm] (bot-max-factor) [above of=bot-factor] {};
\node[maxEdgefactor] [yshift = -7mm] (top-max-factor) [above of=top-factor] {};
\node[maxEdgefactor] [yshift = -7mm] (left-max-factor) [above of=left-factor] {};
\node[maxEdgefactor] [yshift = -7mm] (right-max-factor) [above of=right-factor] {};
\node[maxfactor] [yshift = 3mm] (max-factor) [below left of=top-max-factor]{};

\draw (left-bot) to (bot-factor) to (right-bot) to (right-factor) to (right-top) to
(top-factor) to (left-top) to (left-factor) to (left-bot);
\draw[draw=bpotentialcolor] (left-bot) to (bot-max-factor) to (right-bot) to (right-max-factor) to (right-top) to
(top-max-factor) to (left-top) to (left-max-factor) to (left-bot);
\draw[draw=bfactorcolor] (left-max-factor) to (max-factor) to (right-max-factor);
\draw[draw=bfactorcolor] (top-max-factor) to (max-factor) to (bot-max-factor);

\node[] (fake-up-arrow-1-tail) [above right of = right-top]{};
\node[] (fake-up-arrow-1-head) [above right of = fake-up-arrow-1-tail]{};
\draw [bend left, ->,very thick, black] (fake-up-arrow-1-tail) to node [auto] {} (fake-up-arrow-1-head) ;
\node[] (fake-down-arrow-1-tail) [below right of = right-top]{};
\node[] (fake-down-arrow-1-head) [below right of = fake-down-arrow-1-tail]{};
\draw [bend right,->,very thick, black] (fake-down-arrow-1-tail) to node [auto] {} (fake-down-arrow-1-head) {};

\node[node] (left-bot-mrf) [below right of=fake-up-arrow-1-head] {\small \small $x_2$};
\node[edgefactor] (bot-factor-mrf) [right of=left-bot-mrf] {};
\node[] (fake-left-factor-mrf) [above right of=left-bot-mrf] {};
\node[edgefactorSheared, rotate around = {90:(fake-left-factor-mrf.center)}] (left-factor-mrf) [above right of=left-bot-mrf] {};
\node[node] (right-bot-mrf) [right of=bot-factor-mrf] {\small $x_3$};
\node[node] (left-top-mrf) [above right of=left-factor-mrf] {\small $x_1$};
\node[edgefactor] (top-factor-mrf) [right of=left-top-mrf] {};
\node[node] (right-top-mrf) [right of=top-factor-mrf] {\small $x_4$};
\node[] (fake-right-factor-mrf) [above right of=right-bot-mrf] {};
\node[edgefactorSheared, rotate around = {90:(fake-right-factor-mrf.center)}] (right-factor-mrf) [above right of=right-bot-mrf] {};
\draw (left-bot-mrf) to (bot-factor-mrf) to (right-bot-mrf) to (right-factor-mrf) to (right-top-mrf) to
(top-factor-mrf) to (left-top-mrf) to (left-factor-mrf) to (left-bot-mrf);

\node[node] (left-bot-mp)  [below right of=fake-down-arrow-1-head] {\small \small $y_2$};
\node[] (fake-left-factor-mp) [above right of=left-bot-mp] {};
\node[] (fake-bot-factor-mp) [right of=left-bot-mp] {};
\node[node] (right-bot-mp) [right of=fake-bot-factor-mp] {\small $y_3$};
\node[node] (left-top-mp) [above right of=fake-left-factor-mp] {\small $y_1$};
\node[] (fake-top-factor-mp) [right of=left-top-mp] {};
\node[node] (right-top-mp) [right of=fake-top-factor-mp] {\small $y_4$};
\node[] (fake-right-factor-mp) [above right of=right-bot-mp] {};

\node[maxEdgefactor] [yshift = -7mm] (bot-max-factor-mp) [above of=fake-bot-factor-mp] {};
\node[maxEdgefactor] [yshift = -7mm] (top-max-factor-mp) [above of=fake-top-factor-mp] {};
\node[maxEdgefactor] [yshift = -7mm] (left-max-factor-mp) [above of=fake-left-factor-mp] {};
\node[maxEdgefactor] [yshift = -7mm] (right-max-factor-mp) [above of=fake-right-factor-mp] {};
\node[maxfactor] [yshift = 3mm] (max-factor-mp-f) [below left of=top-max-factor-mp]{};

\draw[draw=bpotentialcolor] (left-bot-mp) to (bot-max-factor-mp) to (right-bot-mp) to (right-max-factor-mp) to (right-top-mp) to
(top-max-factor-mp) to (left-top-mp) to (left-max-factor-mp) to (left-bot-mp);
\draw[draw=bfactorcolor] (left-max-factor-mp) to (max-factor-mp-f) to (right-max-factor-mp);
\draw[draw=bfactorcolor] (top-max-factor-mp) to (max-factor-mp-f) to (bot-max-factor-mp);
\draw[dotted] (right-bot-mrf) to (right-bot-mp);
\draw[dotted] (left-top-mrf) to (left-top-mp);
\draw[dotted] (right-top-mrf) to (right-top-mp);
\draw[dotted] (left-bot-mrf) to (left-bot-mp);

\node[] (fake-up-arrow-2-tail) [below right of = right-top-mrf]{};
\node[] (fake-up-arrow-2-head) [right of = fake-up-arrow-2-tail]{};
\draw [->,very thick, black] (fake-up-arrow-2-tail) to node [auto] {} (fake-up-arrow-2-head) ;
\node[] (fake-down-arrow-2-tail) [below right of = right-top-mp]{};
\node[] (fake-down-arrow-2-head) [right of = fake-down-arrow-2-tail]{};
\draw [->,very thick, black] (fake-down-arrow-2-tail) to node [auto] {} (fake-down-arrow-2-head) {};

\node[node] (left-bot-mrf-1) [above right of=fake-up-arrow-2-head] {\small $x_2^1$};
\node[edgefactor] (bot-factor-mrf-1) [right of=left-bot-mrf-1] {};
\node[] (fake-left-factor-mrf-1) [above right of=left-bot-mrf-1] {};
\node[edgefactorSheared, rotate around = {90:(fake-left-factor-mrf-1.center)}] (left-factor-mrf-1) [above right of=left-bot-mrf-1] {};
\node[node] (right-bot-mrf-1) [right of=bot-factor-mrf-1] {\small $x_3^1$};
\node[node] (left-top-mrf-1) [above right of=left-factor-mrf-1] {\small $x_1^1$};
\draw (left-top-mrf-1) to (left-factor-mrf-1) to (left-bot-mrf-1) to (bot-factor-mrf-1) to (right-bot-mrf-1);

\node[node] (right-bot-mrf-2) [below of=right-bot-mrf-1] {\small $x_3^2$};
\node[] (fake-right-factor-mrf-2) [above right of=right-bot-mrf-2] {};
\node[edgefactorSheared, rotate around = {90:(fake-right-factor-mrf-2.center)}] (right-factor-mrf-2) [above right of=right-bot-mrf-2] {};
\node[node] (right-top-mrf-2) [above right of=right-factor-mrf-2] {\small $x_4^2$};
\node[edgefactor] (top-factor-mrf-2) [left of=right-top-mrf-2] {};
\node[node] (left-top-mrf-2) [left of=top-factor-mrf-2] {\small $x_1^2$};	
\draw (right-bot-mrf-2) to (right-factor-mrf-2) to (right-top-mrf-2) to (top-factor-mrf-2) to (left-top-mrf-2);

\node[node] (right-bot-mp-1) [yshift=-12mm] [below of=right-bot-mrf-2] {\small $y_3^1$};
\node[maxEdgefactor] [yshift = 3mm] (right-factor-mp) [above right of=right-bot-mp-1] {};
\node[node] (right-top-mp-1)  [yshift = -3mm] [above right of=right-factor-mp] {\small $y_4^1$};
\draw[draw=bpotentialcolor] (right-bot-mp-1) to (right-factor-mp) to (right-top-mp-1);

\node[] (fake)  [left of=right-bot-mp-1] {};
\node[node] (left-bot-mp-2) [left of=fake] {\small $y_2^2$};
\node[maxEdgefactor] [yshift = 3mm] (left-factor-mp) [above right of=left-bot-mp-2] {};
\node[node] (left-top-mp-2) [yshift = -3mm] [above right of=left-factor-mp] {\small $y_1^2$};
\draw[draw=bpotentialcolor] (left-bot-mp-2) to (left-factor-mp) to (left-top-mp-2);

\node[node] (left-top-mp-3) [yshift=0mm] [below of = left-top-mp-2] {\small $y_1^3$};
\node[maxEdgefactor] [yshift = 3mm] (top-factor-mp) [right of=left-top-mp-3]{};
\node[node] (right-top-mp-3) [yshift = -3mm] [right of = top-factor-mp] {\small $y_4^3$};
\draw[draw=bpotentialcolor] (left-top-mp-3) to (top-factor-mp) to (right-top-mp-3);

\node[] (fake)  [below left of=left-top-mp-3] {};
\node[node] (left-bot-mp-4) [below left of=fake] {\small $y_2^4$};
\node[maxEdgefactor] [yshift = 3mm] (bot-factor-mp) [right of=left-bot-mp-4] {};
\node[node] (right-bot-mp-4) [yshift = -3mm] [right of=bot-factor-mp] {\small $y_3^4$};
\draw[draw=bpotentialcolor] (left-bot-mp-4) to (bot-factor-mp) to (right-bot-mp-4);

\node[maxfactor] [yshift = 3mm] (max-factor-mp) [ right of=left-factor-mp]{};
\draw[draw=bfactorcolor] (left-factor-mp) to (max-factor-mp) to (right-factor-mp);
\draw[draw=bfactorcolor] (top-factor-mp) to (max-factor-mp) to (bot-factor-mp);

\draw[dotted] (right-bot-mrf-1) to (right-bot-mrf-2) to (right-bot-mp-1) to (right-bot-mp-4);
\draw[dotted] (left-top-mrf-1) to (left-top-mrf-2) to (left-top-mp-2) to (left-top-mp-3);
\draw[dotted] (right-top-mrf-2) to (right-top-mp-1) to (right-top-mp-3);
\draw[dotted] (left-bot-mrf-1) to (left-bot-mp-2) to (left-bot-mp-4);


\end{tikzpicture}
    }
	\caption[]{Illustration of the decomposition~\eqref{eq:bottleneck-labeling-problem-decomposition} for the bottleneck labeling problem.
		Black squares \tikz \node [rectangle, draw=black, fill=black, scale = 0.7]{}; represent MRF-potentials $\theta_{ij}$.
		Orange diamonds \tikz \node [diamond, draw=bpotentialcolor, fill=white, scale = 0.4]{}; stand for bottleneck potentials $\phi_{ij}$.
		The blue star \tikz \node [star, draw=bfactorcolor, fill=bfactorcolor, scale = 0.5]{}; stands for bottleneck costs $\zeta$.
		The problem is decomposed into a MRF part (upper layer) and a bottleneck part (lower layer).
		In turn, the MRF-layer is decomposed into trees and the bottleneck part into chains.
		Bottleneck chains are connected by a global bottleneck term \tikz \node [star, draw=bfactorcolor, fill=bfactorcolor, scale = 0.5]{};.
	}
	\label{fig:grid-decomposition}
\end{figure*}
Since MAP-MRF is NP-hard and the bottleneck labeling problem generalizes it, we cannot hope to obtain an efficient exact algorithm for the general case.
As was done for MAP-MRF~\cite{WernerLocalPolytope,KomodakisDualDecompositionMRF,storvik2000lagrangian}, we approach the general case with a Lagrangian decomposition into tractable subproblems.
To this end, we decompose the underlying MRF problem into trees, which can be solved via dynamic programming.
The bottleneck subproblem is decomposed into a number of bottleneck chain labeling problems. To account for the global bottleneck term, these bottleneck chain problems are connected through a unary bottleneck labeling problem defined on a higher level graph. We use Algorithms~\algref{alg:chain-bottleneck-labeling-problem} and~\algref{alg:unary-bottleneck-labeling-problem} as subroutines to solve the bottleneck decomposition. An example of our decomposition can be seen in Figure~\ref{fig:grid-decomposition}. \\
To account for the MRF-inference problem, we cover the graph $G$ by trees $\SG_1 = (\SV_1,\SE_1),\ldots,(\SV_h,\SE_h)$.
For the bottleneck labeling problem we cover the graph $G$ by chains (trees) $\FG_1=(\FV_1,\FE_1),\ldots,\FG_k = (\FV_k,\FE_k)$.
For the decomposition we introduce variables $\bx^t$ to specify the labeling of the MRF subproblem for graph $\SG_t$, and variables $\by^l$ for the chain bottleneck labeling subproblems defined for graph $\FG_l$.
We propose the following overall decomposition:
\begin{subequations}
	\label{eq:bottleneck-labeling-problem-decomposition}
	\begin{align}
	\min\limits_{\substack{\bx, \{\bx^t\}, \{\by^l\},\\ b \in B}} & \btheta(\bx) + \zeta(b)& \\
	\text{s.t.} \qquad &x_f^t = x_f \quad &\forall f \in \SV_t \cup \SE_t, t \in [h] \label{eq:bottleneck-labeling-problem-decomposition-t} \\
	&y_f^l = x_f \quad &\forall f \in \FV_l \cup \FE_l, l \in [k] \label{eq:bottleneck-labeling-problem-decomposition-l} \\
	&\bx^t \in \SX_{\SV_t} \quad &\forall t \in [h] \\ 
	&\by^l \in \SX_{\FV_l} \quad &\forall l \in [k] \label{eq:bottleneck-labeling-problem-decomposition-y}\\
	&\phi_f(y_f^l) \leq b \quad &\forall f \in \FV_l \cup \FE_l, l \in [k]  \label{eq:bottleneck-labeling-problem-decomposition-b}
	\end{align}
\end{subequations}
We constrain the variables for MRF tree subproblems and chain bottleneck subproblems to be consistent via~\eqref{eq:bottleneck-labeling-problem-decomposition-t} and~\eqref{eq:bottleneck-labeling-problem-decomposition-l}.
The labelings on chains ($\by$) are required to be feasible with respect to the bottleneck value $b$ via~\eqref{eq:bottleneck-labeling-problem-decomposition-b}.\\
To obtain a tractable optimization problem, we dualize the constraints \eqref{eq:bottleneck-labeling-problem-decomposition-t}  and \eqref{eq:bottleneck-labeling-problem-decomposition-l} using dual variables $\lambda^t_e$ and $\eta^l_e$ respectively. 
We keep rest of the constraints and denote the feasible variables $\by^l$ for chain $l$ w.r.t constraints \eqref{eq:bottleneck-labeling-problem-decomposition-y}, \eqref{eq:bottleneck-labeling-problem-decomposition-b} by the set $Y^l(b)$ as:
\begin{equation}
Y^l(b) = \left\{\by^l \in \SX_{\FV_l}\; \middle| \phi_f(y_f^l) \leq b, \; \forall f \in \FV_l \cup \FE_l
\right\}
\label{eq:feasible-chain-labeling}
\end{equation}  
The dual problem is:
\begin{subequations}
	\begin{gather}
	\max\limits_{\blambda,\etab} \left[\sum\limits_{t \in [h]}  E^t(\blambda^t) + J(\etab) \right] \label{eq:dual-bottleneck-labeling-objective}\\
	\text{s.t.} \sum\limits_{\substack{t \in [h]: \\ f \in \SV_t \cup \SE_t}} \lambda^t_f + \sum\limits_{\substack{l \in [k]: \\f \in \FV_l \cup \FE_l}} \eta_f^l = \theta_f \quad \forall f \in V \cup E \label{eq:reparamterization-constraint}
	\end{gather}
\end{subequations}
Where $E^t(\blambda^t)$ and $J(\etab)$ are defined as:
\begin{subequations}
\begin{align}
E^t(\blambda^t) &:= \min\limits_{\bx^t \in \SX_{\SV_t}} \la \blambda^t, \bx^t \ra \label{eq:E_t}\\
J(\etab) &:= \min\limits_{b \in B} \left[\zeta(b) + \sum\limits_{l \in [k]} \min\limits_{y^l \in Y^l(b)} \la \etab^l, \by^l \ra \right]  \label{eq:J_eta}
\end{align}
\end{subequations}
Evaluating $E^t(\cdot)$ amounts to solving a MAP-MRF problem on a tree.
Evaluating $J(\cdot)$ corresponds to computing for each bottleneck value $b$ the corresponding minimal assignment from the set $Y^l(b)$ for all chains $l$ in $[k]$, and then choosing a bottleneck value such that the sum over all chain subproblems is minimal. \\
The dual problem~\eqref{eq:dual-bottleneck-labeling-objective} is a non-smooth concave maximization problem.
Evaluating $E^t(\cdot)$ and $J(\cdot)$ gives supergradients of~\eqref{eq:dual-bottleneck-labeling-objective} which can be used in subgradient based solvers (e.g. subgradient descent, bundle methods).

\myparagraph{Finding $E^t(\blambda^t)$:}
Each MRF tree subproblem $E^t(\blambda^t)$ can be solved independently using dynamic programming to get the optimal labeling $\bar{\bx}^t$.

\myparagraph{Finding $J(\etab)$:}
$J(\etab)$ is solved by Algorithm~\algref{alg:multiple-chain-bottleneck-labeling-problem}.
It proceeds as follows:
\begin{enumerate}[noitemsep,topsep=0pt,parsep=0pt,partopsep=0pt,wide=\parindent]
	\item Build a higher level graph $H = ([k], \varnothing)$ and represent each chain in $[k]$ as a node in $H$. Populate the potentials $M_l$ for all nodes $l \in [k]$ by going through all $b \in B$ using Algorithm~\algref{alg:chain-bottleneck-labeling-problem} (lines~\ref{alg:multiple-chain-bottleneck-labeling-problem-build-higher-level-graph-begin}-\ref{alg:multiple-chain-bottleneck-labeling-problem-build-higher-level-graph-end}).
	\item Use Algorithm~\algref{alg:unary-bottleneck-labeling-problem} to find an optimal bottleneck value $b^*$ of $b \in B$ in graph $H$ (lines~\ref{alg:multiple-chain-bottleneck-labeling-problem-solve-higher-level-graph-begin}-\ref{alg:multiple-chain-bottleneck-labeling-problem-solve-higher-level-graph-end}). 
	\item Find optimal labelings $\overline{y}$ for each chain subproblem by disallowing the configurations having bottleneck values more than $b^*$. (lines~\ref{alg:multiple-chain-bottleneck-labeling-problem-find-chain-labeling-begin}-
	\ref{alg:multiple-chain-bottleneck-labeling-problem-find-chain-labeling-end}).
\end{enumerate} 
\begin{algorithm}[h]
\SetAlgorithmName{chain\_decomp}{chain\_decomp}{List of algorithms}
	\caption{bottleneck labeling on chain decomposition of general graph}
	\KwData{
	    Chain MRFs: $\{R_l = (\FG_l, \SX_l, \etab^l)\}_{l \in [k]}$, \\
		Bottleneck potentials on chains: $\{\bphi^l\}_{l \in [k]}$ \\
	}
	\KwResult{Optimal solution $\{\by^l\}_{l \in [k]}$ and $b^*$ for $J(\etab)$}
	$H \leftarrow ([k], \varnothing)$ \hfill \tcp{higher level graph}	\label{alg:multiple-chain-bottleneck-labeling-problem-build-higher-level-graph-begin}
	\For{$l \in [k]$} {
		\tcp{Populate potentials for node $l$:}
		$M_l \leftarrow ~\algref{alg:chain-bottleneck-labeling-problem}(R_l, \bphi^l)$\;
		\For {$\forall (b,c) \in M_l$} {
			$({\Phi_l}(b), {\Theta_l}(b)) = (b, c)$ \ \tcp{Bott., MRF pots.}
			$\SY_l = \SY_l \cup b$ \hfill \tcp{Add $b$ as a label for $l$}
		}
	}
	\label{alg:multiple-chain-bottleneck-labeling-problem-build-higher-level-graph-end}
	\tcp{Solve graph $H$: }
	$\overline{M} \leftarrow ~\algref{alg:unary-bottleneck-labeling-problem}((H, \SY, \bTheta), \bPhi)$\;
	\label{alg:multiple-chain-bottleneck-labeling-problem-solve-higher-level-graph-begin}
	$(b^*,c^*) = \argmin_{(b,c) \in \overline{M}} \zeta(b) + c$\;
	\label{alg:multiple-chain-bottleneck-labeling-problem-solve-higher-level-graph-end}
	\tcp{Optimal labeling of chains:}
	\For{$l \in [k]$ 	\label{alg:multiple-chain-bottleneck-labeling-problem-find-chain-labeling-begin}} {	
		$(\FD_l = (W_l, \FA_l), \bsigma^l, \bomega^l) \leftarrow ~\algref{alg:chain-digraph-conversion}(R_l, \bphi^l)$
		\For{$(n_i, n_j) \in \FA_l$}
		{
			\If{$\omega^l(n_i, n_j) > b^*$}{
			    $\sigma^l(n_i, n_j) \leftarrow \infty$ 	\tcp{mark as infeasible}
			}
		}
		$\overline{\by}^l = \text{shortest $s,t$-path in } (W_l,A_l,\bsigma^l)$ \;  
	}
	\label{alg:multiple-chain-bottleneck-labeling-problem-find-chain-labeling-end}
	\mylabel{alg:multiple-chain-bottleneck-labeling-problem}{\textbf{chain\_decomp}}
\end{algorithm}

We optimize the Lagrange multipliers in the dual problem~\eqref{eq:dual-bottleneck-labeling-objective}
subject to constraints \ref{eq:reparamterization-constraint} with the proximal bundle method~\cite{FWMAP}. After solving the dual problem we have found a valid reparameterization of the MRF potentials~\eqref{eq:reparamterization-constraint} which maximizes the lower bound. This lower bound can be used as a measure of the quality of the primal solution and helps in rounding a primal solution.

\subsection{Primal rounding}
\label{subsection:primal-rounding}
For recovering the primal solution through rounding on MRF subproblems, we use the approach of~\cite{CTRWS} by treating the dual optimal values of $\blambda$ as MRF potentials. Additional details on how to best choose $\blambda$ and more details are given in Section~\ref{sec:primal-rounding-appendix} in the Appendix. \\
We also considered using message passing instead of subgradient ascent by computing the min-marginals through the procedure used in primal rounding. However, our strong relaxation does not easily allow an efficient way for min-marginal computation. Specifically, for the subproblem $J(\etab)$ in~\eqref{eq:J_eta} it is hard to reuse computations to recalculate min-marginals efficiently for different variables, making any message passing approach slow. Rounding is only executed once at the end, hence very efficient min-marginal computation is of lesser concern.


\section{Experiments}
\label{sec:experiments}
\begin{figure}[h]
	\centering
	\includegraphics[width=0.4\textwidth]{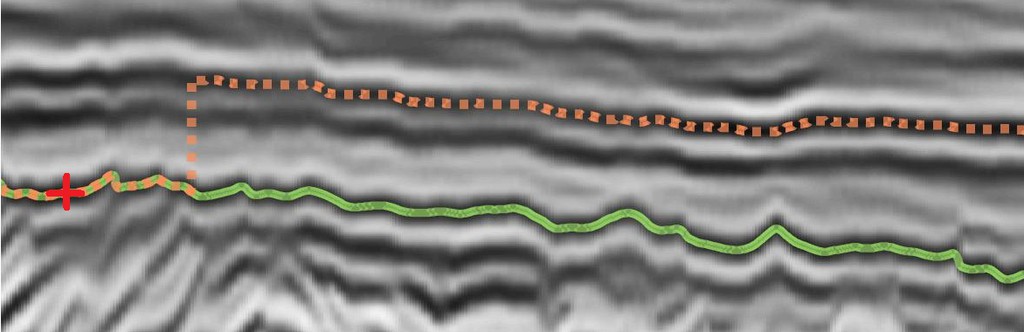}
	\caption[]{
        Exemplary failure case for conventional MRF (\textcolor{orange}{-} in dotted orange) in seismic horizon tracking problem as compared to MRF with an additional bottleneck potential (\textcolor{green}{-} in solid green). (\textcolor{red}{+})~indicates the seed. The MRF solution makes one local error with high cost and starts tracking another smoother layer leading to an overall lower cost solution. A bottleneck term penalizes such high cost errors and results in the correct track.}
	\label{fig:example_1}
\end{figure}

\begin{table*}
\small
    \centering
    \smallskip
    \setlength\tabcolsep{4.5pt} 
    \begin{tabular}{@{\extracolsep{4pt}}lcccccccrcccc@{}}
    \toprule
    & \multicolumn{6}{c}{\texttt{F3-Netherlands}} & \multicolumn{2}{c}{\texttt{Opunake-3D}} & \multicolumn{4}{c}{\texttt{Waka-3D}} \\
    \cline{2-7} \cline{8-9} \cline{10-13} \noalign{\smallskip}
    & I & II* & III* & IV & V* & VI & I* & II & I & II* & III* & IV \\
    \midrule 
    \multicolumn{13}{c}{Instance sizes} \\
    \midrule
     $\abs{V}$ & 263K & 153K & 362K & 153K & 231K & 101K & 443K & 2965K & 614K & 366K & 601K & 614K\\
      $\sum\limits_{i \in V} \abs{\SX_i}$ & 3271K & 1768K & 4394K & 1422K & 1896K & 944K & 406K & 2630K & 3524K & 2020K & 3110K & 3646K \\
    \midrule
    \multicolumn{13}{c}{Mean absolute deviation from ground truth} \\
\midrule
 \texttt{MST}
 & 1.6212 & 1.1245 &  1.6542 & 0.1283 & 0.1088 &  \textbf{1.5319} &  0.4930 & 0.4596 & 2.0060 & 0.2759 & 0.9457 & 1.6108\\
 \texttt{MRF}
 & 7.1435 & 6.2105 & 3.2169 & 0.1540 & 0.0896 &  2.4498 & \textbf{0.4786} & 0.9002 & 1.4073 & 0.0221 & 0.4245 & 1.2388\\
\texttt{B-MRF}
& \textbf{0.8881} & \textbf{0.2644} & \textbf{0.0889} & \textbf{0.0463} & \textbf{0.0894}  & 1.6049  & 0.4887 & \textbf{0.4209} & \textbf{1.1855} & \textbf{0.0116} & \textbf{0.4135} & \textbf{1.0377} \\
\midrule
\multicolumn{13}{c}{Runtimes (seconds)} \\
\midrule
\texttt{MST}
 & 14 & 7 & 19 & 3 & 6 & 7 & 8 & 8 & 8 & 5 & 7 & 8 \\
\texttt{MRF}
 & 673 & 139 & 931 & 62 & 65 & 33 & 114 & 522 & 740 & 22 & 687 & 781 \\
\texttt{B-MRF}
& 1820 & 1374 & 3108 & 1609 & 1554 & 1322 & 2979 & 2554 & 6548 & 2402 & 7738 & 8158 \\
    \bottomrule 
    \end{tabular}
    \smallskip
\caption{\textit{\textbf{Top:}} Instance sizes of horizon surfaces. $\abs{V}$ represents total number of nodes and $\sum\limits_{i \in V} \abs{\SX_i}$ the total number of labels in each instance. 
\textit{\textbf{Middle:}} Mean absolute deviation from the ground-truth. \textit{\textbf{Bottom:}} Runtimes using \texttt{MST}~\cite{MST}, TRWS~\cite{CTRWS} solver for \texttt{MRF}, and our method \texttt{B-MRF}. * marks the horizon surfaces used to train the CNN for computing patch similarity~\eqref{eq:cnn-probability}.}
    \label{tab:results}
\end{table*}
\begin{figure}
\subfigure[Ground-truth]{\includegraphics[width=0.2\textwidth]{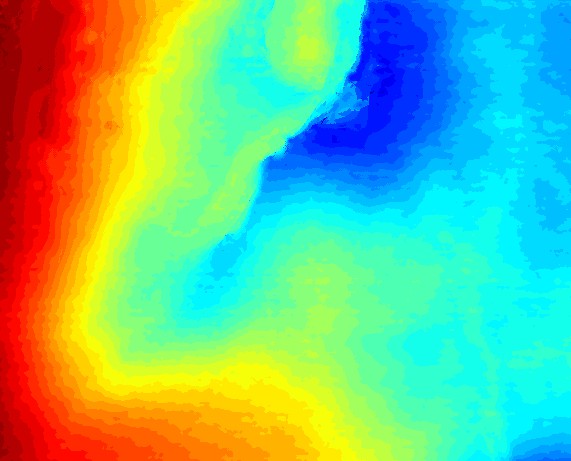}}
\hfill
\subfigure[\texttt{MST} (\textit{MAD=1.6212})]{\includegraphics[width=0.2\textwidth]{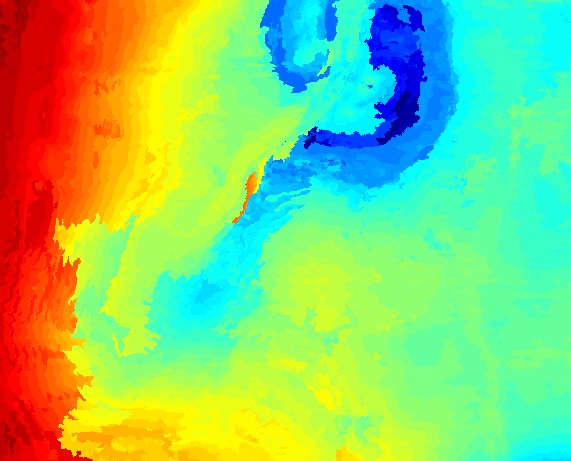}}
\hfill
\subfigure[\texttt{MRF} (\textit{MAD=7.1435})]{\includegraphics[width=0.2\textwidth]{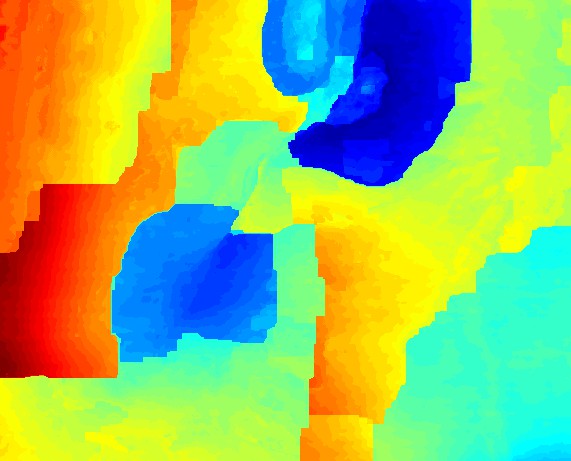}}
\hfill
\subfigure[\texttt{B-MRF} (\textit{MAD=\textbf{0.8881}})]{\includegraphics[width=0.2\textwidth]{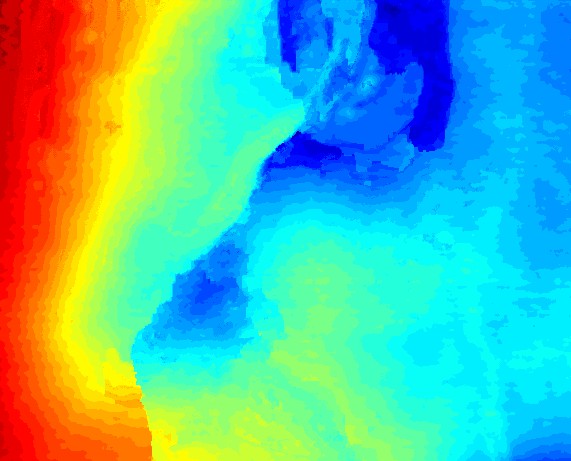}}
\hfill
\caption{Comparison of results for horizon surface \texttt{F3-Netherlands-I}, colors indicate depth of the surface. Mean absolute deviation (\textit{MAD}) scores are used as error metric. The \texttt{MRF} solution is piece-wise smooth but tracks wrong layers, \texttt{MST} is incorrect in the top-right and lower-left region, while our approach of \texttt{B-MRF} has least errors. Best viewed in color.}
\label{fig:F3-I} 
\end{figure}
We apply our proposed technique for tracking layers (horizons) in open source sub-surface volumetric data (a.k.a.\ seismic volumes). Accurate tracking of horizons is one of the most important problems in seismic interpretation for geophysical applications. See Fig.~\ref{fig:example_1} for an illustration of a cross-section of a typical seismic volume. \\
Popular greedy approaches such as~\cite{MST} rely on tracking horizons by establishing correspondences between nearby points that lie on the same horizon with high probability. As such they are prone to fall into local optima. \\
On the other hand a natural option to track horizons is to state the problem as MAP-inference in MRFs for which high-quality solvers exist that are less sensitive to local optimal.
However, the cost structure of MRFs is not adequate for our problem.
Specifically, the MRF energy is composed of a sum of local terms defined over nodes and edges of the graphical model that account for similarity between seed and adjacent  nodes respectively.
For the cost function it may be more advantageous to track a wrong horizon if it is more self-similar and pay a higher dissimilarity cost at a few locations rather than follow the correct horizon if it is harder to follow. This problem is exacerbated by the typically rather weak unary costs (see below).
For an illustration of this behaviour we refer to Figure~\ref{fig:example_1}. \\
We argue that the bottleneck potential remedies this shortcoming of MAP-MRF, while preserving its advantages. In contrast to unary and pairwise potentials of MRFs, the bottleneck potential penalizes a single large discontinuity in the tracking that comes from jumping between layers more than multiple small discontinuities stemming from a hard to track horizon. As such it prefers to follow a rugged horizon with no big jump over a single large jump followed by an easy to track horizon.

Further computational challenges in horizon tracking are:
\begin{itemize}[noitemsep,topsep=0pt,parsep=0pt,partopsep=0pt,wide=\parindent]
	\item Nearby rock layers look very similar to the layer in consideration, due to which tracking algorithms can easily jump to a wrong layer. 
	\item Due to structural deformations in the subsurface environment, rock layers can have discontinuities.
	\item The appearance of a horizon can vary at different locations. Nearby horizon layers can bifurcate or disappear. This makes estimation of reliable cost terms for our model~\eqref{eq:bottleneck-labeling-problem} difficult.
	\item The size of a seismic survey can be huge. Some open-source seismic volumes cover a $3000km^2$ survey area. The corresponding input data has a size of $100$GB.
	\item Lack of established open-source ground truth makes the evaluation procedure and learning of parameters difficult.
\end{itemize}

\myparagraph{Seismic horizon tracking formulation:}
Our input are 3-D seismic volumes of size $N_1 \times N_2 \times D$, where $D$ corresponds to the depth axis and $N_1, N_2$ to x,y-axes resp. 
For each location $(x,y) \in [N_1] \times [N_2]$ in the volume we seek to assign a depth value $z \in [D]$. We formalize this as searching for a labeling function $z: [N_1] \times [N_2] \rightarrow [D]$. \\
For computational experiments we used publicly available seismic volumes~\cite{seismic_volume_webpage}. 
Due to lack of available ground truth, we have selected three volumes (\texttt{F3 Netherlands, Opunake-3D, Waka-3D}) and tracked 11 horizons by hand with the commercial seismic interpretation software~\cite{GVERSE}.
Annotation took two days of work of an experienced seismic interpreter.

\myparagraph{Matching costs:}
In order to track horizons, established state-of-the-art tracking algorithms~\cite{MST,PDE,2d_seismic_horizon_tracking_shortest_path} rely on computing similarity scores between patches that might or might not lie on the same horizon.
This is traditionally done using hand-crafted features that rely on cross-correlation or optical flow.
We opt for a learning-based approach following the great success of CNN-based architectures for computing patch similarity~\cite{patch-compare-CNN, patch-compare-CNN-2}.
Specifically, we extend the architecture of the 2-D patch similarity CNN~\cite{patch-compare-CNN} to efficiently deal with the 3-D structure. CNN architecture and the procedure for training is mentioned in Figure~\ref{fig:cnn} and Section~\ref{sec:CNN-training} in Appendix.
We train two CNNs to compute matching probabilities between depth labels $z_{ij} = z_i, z_j$ of nodes $i = (x_i, y_i), j = (x_j, y_j)$.
The two neural networks $\text{CNN}_n$ and $\text{CNN}_f$ compute matching probabilities between adjacent and non-adjacent nodes respectively:
\begin{equation}
	p_{ij}(z_{ij}) = \text{CNN}_{n|f}(I_i(z_i), I_j(z_j)) \in [0,1]
	\label{eq:cnn-probability}
\end{equation}
In~\eqref{eq:cnn-probability} $I_k(z_k)$, is a two channel image of size $63\times63$ centered around $(x_k, y_k, z_k)$. \\
\myparagraph{Graph construction:}
We rephrase computing a depth labeling as labeling nodes in a graph. 
Consider the grid graph $G=(V,E)$ with nodes $V = [N_1] \times [N_2]$ corresponding to points on $x$ and $y$ axes of the seismic volume. Edges $E$ are given by a $4-$neighborhood induced by the $(x,y)$-grid.
For each node $v \in V$ the label space is $\mathcal{X}_v = [D]$.
A depth labeling thus corresponds to a node labeling of $V$.
We compute unary $\theta_v : [D] \rightarrow \R$, $v \in V$ and pairwise potentials $\theta_{uv} : [D] \times [D] \rightarrow \R$, $uv \in E$ based on the patch matching costs computed in~\eqref{eq:cnn-probability}. Further details about unary and pairwise potentials can be found in Section~\ref{sec:cost-formulation} in Appendix.

\myparagraph{Algorithms:}
We compare our method with a plain MRF, and a state-of-the-art heuristic for horizon tracking based on a variation of minimum spanning tree (MST) problem.
\begin{itemize}[noitemsep,topsep=0pt,parsep=0pt,partopsep=0pt,wide=\parindent]
    \item \texttt{MST:} The horizon tracking approach of~\cite{MST} is a greedy approach inspired by Bor\r{u}vka's algorithm for MST~\cite{boruvka}. It starts by marking the nodes where the seed labels are given and iteratively marks assigns label to the adjacent node that contributes minimal cost w.r.t.\ unary and pairwise potentials.
    \item \texttt{MRF:} We solve MAP-MRF on $G=(V,E)$ with unary and pairwise potentials $\theta$ from above with TRWS~\cite{CTRWS}.
\item \texttt{B-MRF:} We additionally include a bottleneck term on pairwise bottleneck potentials~\eqref{eq:bottleneck-potentials-formula}. 
We solve the problem with the algorithm from Section~\ref{subsection:arbitrary-graphs}, and round a primal solution with the procedure from Section~\ref{subsection:primal-rounding}.
\end{itemize}

\myparagraph{Results:}
Table~\ref{tab:results} lists the mean absolute deviation (MAD) of different methods for tracking horizon surfaces. MAD is computed as $\frac{\|d - \tilde{d}\|_1}{N_1\cdot N_2}$ i.e., the $L_1$ norm of the difference in tracked surface $\tilde{d}$ and the ground-truth $d$, normalized by the dimensions $N_1, N_2$ of the surface. \\
From table~\ref{tab:results} we see that our method \texttt{B-MRF} outperforms \texttt{MRF} and \texttt{MST} by a wide margin on most problems. The resulting depth surfaces can be seen in Figure~\ref{fig:F3-I}, and the rest in the Appendix. 
Inspecting the results we see that \texttt{MRF} finds piece-wise continuous assignments not necessarily corresponding to the correct surface but preferring easy to track ones.
On the other hand, \texttt{MST} often follows the horizon quite well but as soon as it marks a pixel wrongly, it often cannot recover due to its greedy nature.
Our method addresses these shortcomings to a large degree.
The resulting surfaces are smooth and track the correct surface to a larger extent than both \texttt{MST} and \texttt{MRF} by favoring piece-wise smooth regions due to the ordinary unary and pairwise MRF costs, yet following tracks that do not make incorrect jumps due to the additional bottleneck costs. 
Since we optimize a global energy by a convex relaxation method, we suffer less from poor local optimal than \texttt{MST}.
On the other hand, the more expressive optimization model and advanced algorithms that allow \texttt{B-MRF} to obtain higher accuracy also lead to higher runtimes.
	{\small
        \bibliographystyle{ieee_fullname}

		\bibliography{referencesMain}

\begin{thebibliography}{10}\itemsep=-1pt

\bibitem{bakker-thesis}
Peter Bakker.
\newblock {\em Image structure analysis for seismic interpretation}.
\newblock PhD thesis, Delft University of Technology, 2002.

\bibitem{bistarelli2004semirings}
S. Bistarelli.
\newblock {\em Semirings for Soft Constraint Solving and Programming}.
\newblock Lecture Notes in Computer Science. Springer Berlin Heidelberg, 2004.

\bibitem{Bistarelli1999}
S. Bistarelli, U. Montanari, F. Rossi, T. Schiex, G. Verfaillie, and H.
  Fargier.
\newblock Semiring-based {CSP}s and valued {CSP}s: Frameworks, properties, and
  comparison.
\newblock {\em Constraints}, 4(3):199--240, Sep 1999.

\bibitem{boruvka}
Otakar Boruvka.
\newblock { O Jist\'{e}m Probl\'{e}mu Minim\'{a}ln\'{\i}m (About a Certain
  Minimal Problem) (in Czech, German summary)}.
\newblock {\em Pr\'{a}ce Mor. Pr\'{\i}rodoved. Spol. v Brne III}, 3, 1926.

\bibitem{ConvexOptimizationBoyd}
Stephen Boyd and Lieven Vandenberghe.
\newblock {\em Convex Optimization}.
\newblock {Cambridge University Press}, March 2004.

\bibitem{Boykov2006}
Y. Boykov and O. Veksler.
\newblock {\em Graph Cuts in Vision and Graphics: Theories and Applications},
  pages 79--96.
\newblock Springer US, Boston, MA, 2006.

\bibitem{shortestPathBook}
Thomas~T. Cormen, Charles~E. Leiserson, and Ronald~L. Rivest.
\newblock {\em Introduction to Algorithms}.
\newblock MIT Press, Cambridge, MA, USA, 1990.

\bibitem{PowerWaterShed-1}
C. Couprie, L. Grady, L. Najman, and H. Talbot.
\newblock Power watershed: A unifying graph-based optimization framework.
\newblock {\em IEEE Transactions on Pattern Analysis and Machine Intelligence},
  33(7):1384--1399, July 2011.

\bibitem{cuninghame1979minimax}
R.A. Cuninghame-Green.
\newblock {\em Minimax algebra}.
\newblock Lecture notes in economics and mathematical systems. Springer-Verlag,
  1979.

\bibitem{BSP_Mosaicking}
Elena Fernandez, Robert Garfinkel, and Roman Arbiol.
\newblock Mosaicking of aerial photographic maps via seams defined by
  bottleneck shortest paths.
\newblock {\em Oper. Res.}, 46(3):293--304, Mar. 1998.

\bibitem{consistent-labeling}
Boris Flach and Michail~I. Schlesinger.
\newblock A class of solvable consistent labeling problems.
\newblock In Francesc~J. Ferri, Jos{\'e}~M. I{\~{n}}esta, Adnan Amin, and Pavel
  Pudil, editors, {\em Advances in Pattern Recognition}, pages 462--471,
  Berlin, Heidelberg, 2000. Springer Berlin Heidelberg.

\bibitem{bottleneckMST}
Harold~N Gabow and Robert~E Tarjan.
\newblock Algorithms for two bottleneck optimization problems.
\newblock {\em Journal of Algorithms}, 9(3):411 -- 417, 1988.

\bibitem{2d_seismic_horizon_tracking_shortest_path}
Eliana~L. Goldner, Cristina~N. Vasconcelos, Pedro~Mario Silva, and Marcelo
  Gattass.
\newblock A shortest path algorithm for {2D} seismic horizon tracking.
\newblock In {\em Proceedings of the 30th Annual ACM Symposium on Applied
  Computing}, SAC '15, pages 80--85, New York, NY, USA, 2015. ACM.

\bibitem{GVERSE}
{GVERSE Geophysics 2017.3}.
\newblock Lmkr.
\newblock \url{http://www.lmkr.com/gverse/gverse-geophysics/}.

\bibitem{BSP}
Volker Kaibel and Matthias A.~F. Peinhardt.
\newblock On the bottleneck shortest path problem, 2006.

\bibitem{opengm_benchmark_IJCV}
J\"{o}rg~H. Kappes, Bj\"{o}rn Andres, Fred~A. Hamprecht, Christoph Schn\"{o}rr,
  Sebastian Nowozin, Dhruv Batra, Sungwoong Kim, Bernhard~X. Kausler, Thorben
  Kr\"{o}ger, Jan Lellmann, Nikos Komodakis, Bogdan Savchynskyy, and Carsten
  Rother.
\newblock A comparative study of modern inference techniques for structured
  discrete energy minimization problems.
\newblock {\em International Journal of Computer Vision}, 115(2):155--184,
  2015.

\bibitem{semiring_induced_valuation_algebras_Kohlas_Wilson}
J{\"{u}}rg Kohlas and Nic Wilson.
\newblock Semiring induced valuation algebras: Exact and approximate local
  computation algorithms.
\newblock {\em Artif. Intell.}, 172(11):1360--1399, 2008.

\bibitem{KohliCVPR2010}
Pushmeet Kohli and M.~Pawan Kumar.
\newblock Energy minimization for linear envelope {MRF}s.
\newblock In {\em CVPR}, 2010.

\bibitem{CTRWS}
V. Kolmogorov.
\newblock Convergent tree-reweighted message passing for energy minimization.
\newblock {\em IEEE Transactions on Pattern Analysis and Machine Intelligence},
  28(10):1568--1583, Oct 2006.

\bibitem{KolmogrovCSPs}
Vladimir Kolmogorov, Johan Thapper, and Stanislav Z\'{z}ivny\'{z}.
\newblock The power of linear programming for general-valued {CSP}s.
\newblock {\em SIAM J. Comput.}, 44(1):1--36, Feb. 2015.

\bibitem{KomodakisDualDecompositionMRF}
N. Komodakis, N. Paragios, and G. Tziritas.
\newblock {MRF} energy minimization and beyond via dual decomposition.
\newblock {\em IEEE Transactions on Pattern Analysis and Machine Intelligence},
  33(3):531--552, March 2011.

\bibitem{korte1991greedoids}
B.H. Korte, L. Lov{\'a}sz, and R. Schrader.
\newblock {\em Greedoids}.
\newblock Algorithms and Combinatorics. Springer-Verlag, 1991.

\bibitem{liu2009blockwise}
H. Liu, M. Palatucci, and J. Zhang.
\newblock {Blockwise coordinate descent procedures for the multi-task lasso,
  with applications to neural semantic basis discovery}.
\newblock {\em ICML}, pages 649--656, 2009.

\bibitem{network_flow_for_structured_sparsity}
Julien Mairal, Rodolphe Jenatton, Francis~R. Bach, and Guillaume~R Obozinski.
\newblock Network flow algorithms for structured sparsity.
\newblock In J.~D. Lafferty, C.~K.~I. Williams, J. Shawe-Taylor, R.~S. Zemel,
  and A. Culotta, editors, {\em Advances in Neural Information Processing
  Systems 23}, pages 1558--1566. Curran Associates, Inc., 2010.

\bibitem{Pansari_2017_CVPR}
Pankaj Pansari and M. Pawan~Kumar.
\newblock Truncated max-of-convex models.
\newblock In {\em CVPR}, 2017.

\bibitem{paszke2017automatic}
Adam Paszke, Sam Gross, Soumith Chintala, Gregory Chanan, Edward Yang, Zachary
  DeVito, Zeming Lin, Alban Desmaison, Luca Antiga, and Adam Lerer.
\newblock Automatic differentiation in {PyTorch}.
\newblock In {\em NIPS Autodiff Workshop}, 2017.

\bibitem{l1_infinity_icml_2009}
Ariadna Quattoni, Xavier Carreras, Michael Collins, and Trevor Darrell.
\newblock An efficient projection for l1,infinity regularization.
\newblock In Andrea~Pohoreckyj Danyluk, Léon Bottou, and Michael~L. Littman,
  editors, {\em ICML}, volume 382 of {\em ACM International Conference
  Proceeding Series}, pages 857--864. ACM, 2009.

\bibitem{rathke2014probabilistic}
Fabian Rathke, Stefan Schmidt, and Christoph Schn{\"o}rr.
\newblock Probabilistic intra-retinal layer segmentation in {3-D} {OCT} images
  using global shape regularization.
\newblock {\em Medical image analysis}, 18(5):781--794, 2014.

\bibitem{consistent-labeling-2}
Michail~I Schlesinger and Boris Flach.
\newblock Some solvable subclasses of structural recognition problems.
\newblock In {\em Czech Pattern Recognition Workshop}, volume 2000, pages
  55--62, 2000.

\bibitem{min-marginals-Shekhovtsov}
Alexander Shekhovtsov, Christian Reinbacher, Gottfried Graber, and Thomas Pock.
\newblock Solving dense image matching in real-time using discrete-continuous
  optimization.
\newblock In {\em Proceedings of the 21st Computer Vision Winter Workshop
  (CVWW)}, page~13, 2016.

\bibitem{BSPSP}
Tong-Wook Shinn and Tadao Takaoka.
\newblock Combining the shortest paths and the bottleneck paths problems.
\newblock In {\em Proceedings of the Thirty-Seventh Australasian Computer
  Science Conference - Volume 147}, ACSC '14, pages 13--18, Darlinghurst,
  Australia, Australia, 2014. Australian Computer Society, Inc.

\bibitem{InfinityNormSegmentation}
A.~K. Sinop and L. Grady.
\newblock A seeded image segmentation framework unifying graph cuts and random
  walker which yields a new algorithm.
\newblock In {\em 2007 IEEE 11th International Conference on Computer Vision},
  pages 1--8, Oct 2007.

\bibitem{min-max-propagation}
Christopher Srinivasa, Inmar Givoni, Siamak Ravanbakhsh, and Brendan~J Frey.
\newblock Min-max propagation.
\newblock In I. Guyon, U.~V. Luxburg, S. Bengio, H. Wallach, R. Fergus, S.
  Vishwanathan, and R. Garnett, editors, {\em Advances in Neural Information
  Processing Systems 30}, pages 5565--5573. Curran Associates, Inc., 2017.

\bibitem{storvik2000lagrangian}
Geir Storvik and Geir Dahl.
\newblock Lagrangian-based methods for finding map solutions for {MRF} models.
\newblock {\em IEEE Transactions on Image Processing}, 9(3):469--479, 2000.

\bibitem{FWMAP}
P. {Swoboda} and V. {Kolmogorov}.
\newblock {{MAP} inference via Block-Coordinate Frank-Wolfe Algorithm}.
\newblock In {\em CVPR}, 2019.

\bibitem{seismic_volume_webpage}
{The Society of Exploration Geophysicists}.
\newblock Open data wiki.
\newblock \url{https://wiki.seg.org/wiki/Open_data\#3D_land_seismic_data}.

\bibitem{min-max-message-passing}
M Vinyals, KS Macarthur, A Farinelli, SD Ramchurn, and NR Jennings.
\newblock A message-passing approach to decentralized parallel machine
  scheduling.
\newblock {\em COMPUTER JOURNAL}, 57:856--874, 2014.

\bibitem{patch-compare-CNN-2}
Jure \v{Z}bontar and Yann LeCun.
\newblock Stereo matching by training a convolutional neural network to compare
  image patches.
\newblock {\em Journal of Machine Learning Research}, 17(65):1--32, 2016.

\bibitem{3d_seismic_horizon_tracking_path_constraints}
Ke Wang*, Kaihong Wei, Kevin Deal, and David Wilkinson.
\newblock {\em {3D} Seismic horizon extraction with horizon patch constraints},
  pages 1754--1758.
\newblock 2015.

\bibitem{WernerLocalPolytope}
Tom{\'a}{\v s} Werner.
\newblock A linear programming approach to max-sum problem: A review.
\newblock {\em IEEE Trans. Pattern Analysis and Machine Intelligence},
  29(7):1165--1179, July 2007.

\bibitem{Werner-Soft08}
Tom{\'a}{\v s} Werner.
\newblock Marginal consistency: Unifying constraint propagation on commutative
  semirings.
\newblock In {\em Intl.\ Workshop on Preferences and Soft Constraints}, pages
  43--57, September 2008.

\bibitem{Werner-CVWW-2007}
Tom{\'a}{\v s} Werner and Alexander Shekhovtsov.
\newblock Unified framework for semiring-based arc consistency and relaxation
  labeling.
\newblock In {\em 12th Computer Vision Winter Workshop, St. Lambrecht,
  Austria}, pages 27--34. Graz University of Technology, February 2007.

\bibitem{PDE}
Xinming Wu and Dave Hale.
\newblock Horizon volumes with interpreted constraints.
\newblock {\em GEOPHYSICS}, 80(2):IM21--IM33, 2015.

\bibitem{MST}
Yingwei Yu, Cliff Kelley, and Irina Mardanova.
\newblock {\em Automatic horizon picking in 3D seismic data using optical
  filters and minimum spanning tree (patent pending)}, pages 965--969.
\newblock 2012.

\bibitem{patch-compare-CNN}
Sergey Zagoruyko and Nikos Komodakis.
\newblock Learning to compare image patches via convolutional neural networks.
\newblock In {\em The IEEE Conference on Computer Vision and Pattern
  Recognition (CVPR)}, June 2015.

\end{thebibliography}
	}
	\clearpage

\section{Appendix}
 \begin{lemma} 
    \label{lemma:local-polytope-noninteger}
 	The following relaxation over the local polytope $\Lambda$~\eqref{eq:local-polytope} is not tight for bottleneck labeling problem:
 	\begin{equation}
 	\label{eq:Local-Polytope-with-bottleneck}
 	\min_{\substack{\mu \in \Lambda \\ b}} \sum_{i \in V} \la \theta_i , \mu_i \ra + \sum_{ij \in E} \la \theta_{ij}, \mu_{ij} \ra + b
 	\end{equation}
 	\hspace{42pt}s.t.
 	\begin{subequations}
 		\begin{align}
 		&&b \geq & \la \phi_{i}, \mu_{i} \ra, &&\forall {i \in V} \\
 		&&b \geq & \la \phi_{ij}, \mu_{ij} \ra, &&\forall {ij \in E}
 		\end{align}
 	\end{subequations}
 \end{lemma}
 \begin{proof}
 	We give a proof by example as in the Figure \ref{fig:non-integer-chain} with a 3 node binary graphical model only containing pairwise bottleneck potentials with $\epsilon > 0$. 
 	Both integer optimal solutions have cost: $max(a,a+\epsilon) = a+\epsilon$, whereas the optimal solution of the LP relaxation $\mu_{uv}(0,0) = \mu_{uv}(1,1) = \mu_{vw}(0,0) = \mu_{vw}(1,1) = 0.5$ has cost $a+\frac{\epsilon}{2}$.
 	\begin{figure}
		\centering
		\begin{tikzpicture}

	\node[circle,fill] (x11) {}; 
	\node[circle,fill,below of=x11] (x12) {}; 
	\node[draw,inner sep=2mm,label=below:$u$,fit=(x11) (x12)] {};
	
	\node[circle,fill,right=3cm of x11] (x21) {}; 
	\node[circle,fill,below of=x21] (x22) {}; 
	\node[draw,inner sep=2mm,label=below:$v$,fit=(x21) (x22)] {};
	
	\node[circle,fill,right=3cm of x21] (x31) {}; 
	\node[circle,fill,below of=x31] (x32) {}; 
	\node[draw,inner sep=2mm,label=below:$w$,fit=(x31) (x32)] {};
	
	\draw (x11) to (x21);
	\draw (x12) to (x22);
	\node[label=above:{\textcolor{black}{$a$}},fit=(x11) (x22)] {};
	\node[label=below:{\textcolor{black}{$a + \epsilon$}},fit=(x11) (x22)] {};
	
	\draw (x21) to (x31);
	\draw (x22) to (x32);
	\node[label=above:{\textcolor{black}{$a + \epsilon$}},fit=(x21) (x32)] {};
	\node[label=below:{\textcolor{black}{$a$}},fit=(x21) (x32)] {};

\end{tikzpicture}
		\caption[Non-integrality of Local Polytope relaxation for chain bottleneck labeling]{Example of non-integer solution to the Local Polytope relaxation for a chain graphical model containing nodes $\{u, v, w\}$ with 2 labels for each node and containing only the bottleneck pairwise potentials $\phi_{uv}(\cdot, \cdot), \phi_{vw}(\cdot,\cdot)$.}
		\label{fig:non-integer-chain}
 	\end{figure}
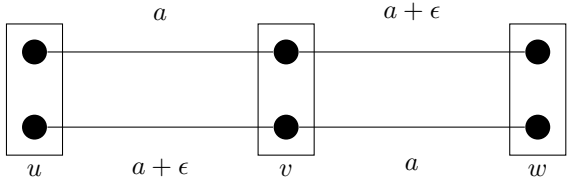
 \end{proof}
 \SetKw{Continue}{continue}

 \begin{algorithm}
 \SetAlgorithmName{chain\_to\_dag}{chain\_to\_dag}{List of algorithms}
 	\caption{chain MRF to directed acyclic graph transformation}
	\KwData {
	    Chain MRF: $(G=(V,E), \SX, \{\theta_f\}_{f \in V \cup E})$, \\
		Bottleneck potentials:  $\{\phi_f(x_f)\} _ {f \in V \cup E}$, \\
    }
	\KwResult
	{
		Directed graph: $D=(W,A)$ \\
		Linear costs: $\sigma(h, t) \quad \forall (h, t) \in A$ \\
		Bottleneck costs: $\omega(h, t) \quad \forall (h, t) \in A$ \\
		Start and terminal nodes $s,t$
	}
    \tcp{Represent each label with two nodes and add source, sink:}
    $W = \{s,t\} \cup \{x_i, \overline{x}_i : i \in V, x_i \in \SX_i\}$\;
    \tcp{Add arcs to represent potentials:}
    $A = \begin{array}{l}
    \{(s, x_{1}) : x_{1} \in \SX_{1}\} \ \cup \\
    \{( x_i,\overline{x}_i) : i \in V, x_i \in \SX_i\} \ \cup \\
    \{\overline{x}_{i}, x_{i+1}) : j\in[n-1], x_{i,i+1} \in \SX_{i,i+1}\}\ \cup \\
    \{( \overline{x}_{n}, t) : x_{n} \in \SX_{n}\}
    \end{array}
    $ \\

\tcp{Unary potentials:}
$\begin{array}{l}
    \sigma(\overline{x}_i,x_i) = \theta_i(x_i) \\
    \omega(\overline{x}_i,x_i) = \phi_i(x_i)
    \end{array}$ \hfill $\forall i \in V, x_i \in \SX_i$\;
    \tcp{Pairwise potentials:}
    \nosemic$\begin{array}{l}
    \sigma(\overline{x}_i,x_{i+1}) = \theta_{i,i+1}(x_{i,i+1}) \\
    \omega(\overline{x}_i,x_{i+1}) = \phi_{i,i+1}(x_{i,i+1})
    \end{array}$\;\pushline\dosemic\nonl$\forall (i,i+1) \in E, x_{i,i+1} \in \SX_{i,i+1}$\;
    
    \tcp{Connect $s$ and $t$:}
	$A' = \cup \begin{array}{c}
	\left\{(s, x_1) : x_1 \in \SX_{1} \right\} \\ 
	\left\{(\overline{x}_n, t) : x_n \in \SX_{n} \right\}
	\end{array}$\;

	\mylabel{alg:chain-digraph-conversion}{\textbf{chain\_to\_dag}}
\end{algorithm}

\begin{algorithm}
\SetAlgorithmName{dsp\_chain}{dsp\_chain}{List of algorithms}
\caption{dynamic shortest path on chains}
 	\KwData{Diagraph: $D = (W,A)$ \\
 		Node distances: $d(w), \quad \forall w \in W$ \\
 		Arc costs: $\sigma(p,q), \quad \forall (p,q) \in A$ \\
 		Arc to insert: $(u,v), \quad u,v \in W$ \\
 	}
 	\KwResult{
 		Updated node distance: $d$ \\
 		Updated nodes: $S \subseteq W$
 	}
 	
 	Initialize $Q = \varnothing$\;
 	$Enqueue(Q, (u,v))$ \;
 	\While{$Q \neq \varnothing$} {
 		$(p,q) = Dequeue(Q)$ \;
 		\If{$d(q) \geq c(p,q) + d(p)$}{
 			$\Continue $\;
 		}
 		$d(q) \coloneqq c(p,q) + d(p)$ \;
 		$S = S \cup \{q\}$ \;
 		\tcp{Enqueue outgoing arcs from $q$ for possible distance update:}
 		\ForEach{$(q, r) \in A$} {
 			$Enqueue(Q, (u,v))$ \;
 		}
 	}
 	\mylabel{alg:dynamic-shortest-path-for-chains}{\textbf{dsp\_chain}}
 \end{algorithm}
 
 \begin{proof}[Lemma~\ref{prop:chain-bottleneck-labeling-problem-runtime}]
Shortest path computation (line~\algref{alg:chain-bottleneck-labeling-problem-shortest-path-update}) on directed acyclic graph $D$ can be done in linear time in the number of arcs $A$ using breadth-first traversal by Algorithm~\algref{alg:dynamic-shortest-path-for-chains}. However, shortest path update is performed every time after introducing each arc in $A$ (lines \ref{alg:chain-bottleneck-labeling-problem-arc-addition-begin}-\ref{alg:chain-bottleneck-labeling-problem-arc-addition-end}), thus making the run-time $\mathcal{O}(|A|^2)$ i.e. quadratic in the number of arcs in $A$. The sorting operation adds an additional factor of $|A|\log{}|A|$.
\end{proof}

\subsection{Experiment details}
\myparagraph{Cost formulation:}
\label{sec:cost-formulation}
For unary potentials of seed node $s$, we set 
$\theta_s(z_s) = \begin{cases} 0,& z_s = z^*_s\\ \infty, & \text{otherwise} \end{cases}$.
Rest of the unary potentials are computed as:
\begin{equation}
	\theta_{i}(z_i) = -\log(\text{CNN}_f(I_s(z_s^*), I_i(z_{i})))
	\label{eq:unary-costs}
\end{equation}
The pairwise potentials contain terms in-addition to patch similarity from the CNN. First, we compute optical flow using the gradient structure tensor based approach of~\cite{bakker-thesis} on the seismic volume along the x and y-axis resulting in optical flow $u_x$ and $u_y$.
From this we compute the displacement penalization for edges in x-direction as
$f_{ij}(z_{ij}) = \abs{z_i - z_j - u_x(x_i,y_i,z_i)}$,
and analogously for edges in y-direction. 
Next, we compute a coherence estimate $C_{ij}$~\cite{bakker-thesis} indicating whether the optical flows $u_x$ and $u_y$ are reliable. 
The pairwise MRF potentials combine the above terms into a sum of appearance terms~\eqref{eq:cnn-probability} and weighted discontinuity penalizers:
\begin{equation}
\theta_{ij}(z_{ij}) = 
\begin{array}{l}
-\log(\text{CNN}_n(I_i(z_i), I_j(z_j))\\ + C_{ij}(z_{ij}) \abs{f_{ij}(z_{ij})}
\end{array}\,
\label{eq:pairwise-costs}
\end{equation} 
The second term allows discontinuities in the horizon surfaces where orientation estimates can be incorrect and penalizes discontinuities where the horizon is probably continuous. 
The authors from~\cite{MST,2d_seismic_horizon_tracking_shortest_path} do not use the coherence estimate, making their cost functions less robust.
The bottleneck pairwise potentials are:
\begin{equation}
\phi_{ij}(z_{ij}) = \abs{E}\ \theta_{ij}(z_{ij})
\label{eq:bottleneck-potentials-formula}
\end{equation}
The scaling parameter $\abs{E}$ in~\eqref{eq:bottleneck-potentials-formula} makes the bottleneck potential invariant to the grid size.

\myparagraph{CNN training:}
\label{sec:CNN-training}
We use six out of eleven labeled horizons for training the CNN adopted from~\cite{patch-compare-CNN} mentioned in Figure~\ref{fig:cnn}. To prevent the CNN from learning the whole ground truth for these six horizons only $10\%$ of the possible patches are used. For data augmentation, translation is performed on non-matching patches as also done in~\cite{patch-compare-CNN-2} for stereo. The validation set contains $2\%$ of the possible patches from each of the eleven horizons. \\ 
For training the $\text{CNN}_n$ in~\eqref{eq:pairwise-costs}, edges in underlying MRF models are sampled for creating the patches. Training was done using PyTorch~\cite{paszke2017automatic} for $200$ epochs  and the model with the best validation accuracy ($95\%$) was used for computing the potentials. \\
For training the $\text{CNN}_f$ used in~\eqref{eq:unary-costs}, we use the same procedure as above except that the patches are sampled randomly and thus their respective nodes do not need to be adjacent. The best validation accuracy on the trained model was $83\%$. \\
\subsection{Primal rounding details}
\label{sec:primal-rounding-appendix}
The function $s(i)$ defines some ordering for all nodes $i$ in $V$.
Based on this order, we sequentially fix labels of the respective node as in Algorithm~\algref{alg:primal-rounding-MRF}.
\begin{algorithm}
\SetAlgorithmName{primal\_rounding}{primal\_rounding}{List of algorithms}
	\caption{primal rounding based on MRF subproblem~\cite{CTRWS}}
	\KwData {
		MRF: $(G = (V,E), \SX, \blambda)$, \\ 
		Ordering of nodes in $G$: $s(v), \quad \forall v \in V$.
	}
	\KwResult
	{
		Primal labeling: $\hat{\bx} \in \SX_V$.
	}
	$V_o \coloneqq V$ \;
	Sort nodes in $V_o$ w.r.t ordering $s(v)$ \;
	$V_l \coloneqq \varnothing, \quad $ \tcp{Set of labeled nodes}
	\For{$i \in V_o$}
	{
		\tcc{Assign label to node $i$ in accordance with the nodes already labeled:}
		$\hat{x}_i \coloneqq \argmin\limits_{x_i \in \SX_i}\left[ \lambda_i(x_i) + \sum\limits_{j \in V_l:ij \in E}\lambda_{ij}(x_i, \hat{x}_j)\right]$ \;
		$V_l = V_l \cup \{i\}, \quad$ \tcp{Mark $i$ as labeled}
	}
	\mylabel{alg:primal-rounding-MRF}{\textbf{primal\_rounding}}
\end{algorithm}

For any given dual variables, there is a whole set of equivalent dual variables that give the same dual lower bound. 
However, when rounding with~\algref{alg:primal-rounding-MRF}, the final solution is sensitive to the choice of dual equivalent dual variables.
One of the reasons is that the rounding is done on the MRF subproblems only, so we want to ensure that the dual variables $\lambda$ carry as much information as possible.
Therefore, we propose a propagation mechanism that modifies the dual variables $\etab$ of the bottleneck potentials such that the overall lower bound is not diminished, while at the same time making the MRF dual variables as informative as possible.
This procedure  is described in Algorithm~\algref{alg:min-marginals-chains}. \\
\begin{algorithm}
    \SetAlgorithmName{Min\_Marginals\_BMRF}{Min\_Marginals\_BMRF}{List of algorithms}
	\caption{Min marginals for bottleneck MRF subproblems}
	\KwData{
		Bottleneck chain graphs: $\{\FG_l\}_{l \in [k]}$, \\
		Linear potentials on chains: $\{\etab^l\}_{l \in [k]}$, \\
		Bottleneck potentials on chains: $\{\bphi^l\}_{l \in [k]}$, \\
		Chain to compute min-marginals: $u \in [k]$, \\
		Costs in higher level graph $\overline{H} = ([k]\setminus \{u\}, \varnothing)$ : \\
		\qquad Bottleneck costs: $\overline{b}$, \quad Linear costs: $\overline{c}$ 
	}
	\KwResult{
		Min-marginals of nodes in $u$: $m_i(\overline{y}_i^u) =
		\min\limits_{b \in B} \left[ \begin{array}{l} \zeta(b) + 
			\min\limits_{\substack{y^u \in Y^u(b):\\ \overline{y}_i^u = y_i^u}} \la \etab^u, \by^u \ra  \\ +
			\sum\limits_{l \in [k]\setminus \{u\}} \min\limits_{y^l \in Y^l(b)} \la \etab^l, \by^l \ra 
		 \end{array} \right]$, \\
		\flushright $\forall i \in \FV_i, \overline{y}_i^u \in \SX_i$
	}
	\tcp{Initialize min-marginals:}
	$m_i(y_i) \coloneqq \infty, \quad \forall i \in \FV_u, y_i \in \SX_i$ \;
	\tcp{Represent chain $u$ as DAG:}
	$(D=(W,A), \bsigma, \bomega) \leftarrow ~\algref{alg:chain-digraph-conversion}(\FG_u, \etab^u, \bphi^u)$ \;
	\tcp{Forward (source-to-sink) shortest path structures:}
	$A'_r = \cup \begin{array}{c}
	\left\{(s, x_1) : x_1 \in \SX_{1} \right\} \\ 
	\left\{(\overline{x}_n, t) : x_n \in \SX_{n} \right\}
	\end{array}$ \\
	$d_r(s) \coloneqq 0$, $d_r(w) \coloneqq \infty, \forall w \in W\setminus\{s\}$\;
	\tcp{Backward (sink-to-source) shortest path structures:}
	$A'_l = \cup \begin{array}{c}
	\left\{(x_1, s) : x_1 \in \SX_{1} \right\} \\ 
	\left\{(t, \overline{x}_n) : x_n \in \SX_{n} \right\}
	\end{array}$ \\
	$d_l(t) \coloneqq 0$, \quad $d_l(w) \coloneqq \infty, \forall w \in W\setminus\{t\}$\;
	Sort $A$ according to values $\bomega$\;
	\For{$(n_i, n_j) \in A$ in ascending order} {
		$A' = (n_i, n_j) \cup A' $ \;
		\tcp{Forward update:}
		$(d_r, S_r) = ~\algref{alg:dynamic-shortest-path-for-chains}((W,A'_r),d_r,\bsigma,(n_i, n_j))$\; 
		\tcp{Backward update:}
		$(d_l, S_l) = ~\algref{alg:dynamic-shortest-path-for-chains}((W,A'_l),d_l,\bsigma,(n_j, n_i))$ \;
		\For{$\{n = (v, y_v)\} \in S_r \cup S_l$}{
			\tcp{Update min-marginal of node $v \in \FV_u$, label $y_v \in \SX_v$:}
			$(b_v, c_v) \leftarrow ~\algref{alg:unary-bottleneck-labeling-problem} ($ \\ $(\{\overline{H}, n\},\varnothing), \{\overline{b}, \omega(n))\},\{\overline{c}, d_r(n) + d_l(n)\} )$\;\label{alg:min-marginals-chains-update-start}
			$m_v(y_v) \coloneqq \min\left(m_v(y_v), \zeta(b_v) + c_v\right)$\;
			\label{alg:min-marginals-chains-update-end}
		}
	}
	\mylabel{alg:min-marginals-chains}{\textbf{Min\_Marginals\_BMRF}}
\end{algorithm}
Schemes similar to Algorithm~\algref{alg:min-marginals-chains} were also proposed in~\cite{min-marginals-Shekhovtsov} for exchanging information between pure MRF subproblems. Intuitively, the goal in such schemes is to extract as much information out of a source subproblem (in our case bottleneck MRF subproblem) and send it to the target subproblem (in our case MRF subproblem)  such that the overall lower bound does not decrease. Such a strategy helps the target subproblem in taking decisions which also comply with the source subproblem. \\
Algorithm~\algref{alg:min-marginals-chains} computes min-marginals of nodes for a given bottleneck chain subproblem $u$. Proceeding in a similar fashion as Algorithm~\algref{alg:chain-bottleneck-labeling-problem}, min-marginals computation on a chain $u$ is done as follows:
\begin{enumerate}[noitemsep,topsep=0pt,parsep=0pt,partopsep=0pt,wide=\parindent]
	\item As an input, the costs of solution in all other chains excluding $u$ are required, which can be obtained in the similar fashion as was done before in Algorithm ~\algref{alg:multiple-chain-bottleneck-labeling-problem} (lines ~\algref{alg:multiple-chain-bottleneck-labeling-problem-build-higher-level-graph-begin}-\ref{alg:multiple-chain-bottleneck-labeling-problem-solve-higher-level-graph-end}).
	\item The algorithm maintains forward and backward shortest path distances $d_r, d_l$  (i.e., distances from source, sink resp.). This helps in finding the cost of minimum distance path passing through any given node in the directed acyclic graph of chain $u$. 
	\item Similar to Algorithm~\algref{alg:chain-bottleneck-labeling-problem}, bottleneck potentials are sorted and the bottleneck threshold is relaxed minimally on each arc addition. 
	\item On every arc addition, the set of nodes for which distance from source/sink got decreased are maintained in $S_r, S_l$ resp. Only this set of nodes will need to re-compute their min-marginals. 
\end{enumerate}
The above-mentioned Algorithm~\algref{alg:min-marginals-chains} only calculates the min-marginals for nodes, the calculation of min-marginals for edges can be carried out in similar way by also updating those edges of the DAG whose adjacent node gets updated (lines ~\algref{alg:min-marginals-chains-update-start}-~\algref{alg:min-marginals-chains-update-end}). After computing the min-marginals, messages to MRF tree Lagrangians $\blambda$ can be computed by appropriate normalization.

\begin{figure*}[h]
	\centering
	\includegraphics[width=0.9\textwidth]{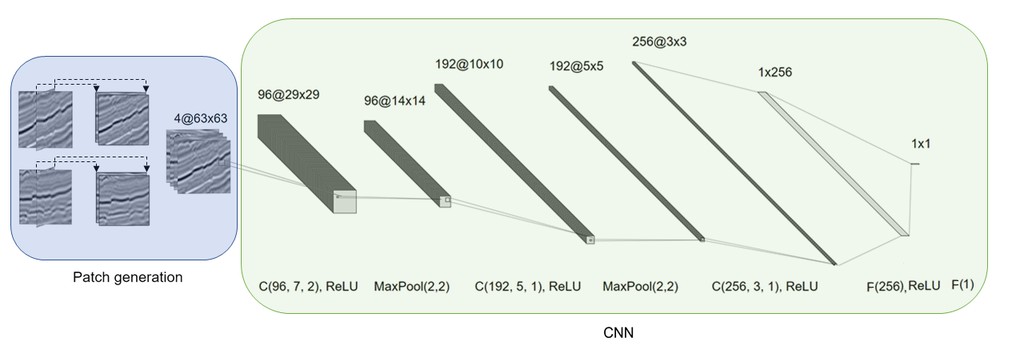}
	\caption[]{Workflow to compute the patch matching probabilities \eqref{eq:cnn-probability}. Two axis-aligned patches are extracted around each voxel in seismic volume to get 2-channel patch image. These two images are concatenated to get a 4-channel image which would be used as an input to the CNN~\cite{patch-compare-CNN} for computing patch matching probabilities by softmax activation at the output. C($n$, $k$, $s$) denotes a convolutional layer having $n$ filters of size $k\times k$ with stride $s$, F($n$) to a fully connected layer with $n$ outputs, and MaxPool($m, n$) to max-pooling with kernel size $m \times m$ and stride $n$. }
	\label{fig:cnn}
\end{figure*}

\subsection{Results}
Figures~\ref{fig:F3-II}-~\ref{fig:waka-IV} contain the visual comparison of tracked horizon surfaces mentioned in Table~\ref{tab:results}. Similar to Figure ~\ref{fig:F3-I}, the color of the surface denotes depth.  

\edef\i{1}
\foreach \h[count=\j] in {II,III,IV,V,VI} {
\pgfmathparse{int(\i+1)} 
\xdef\i{\pgfmathresult}
\begin{figure*}[h]
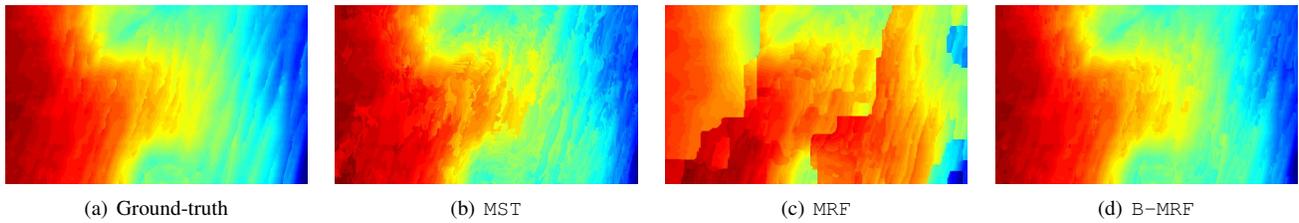
 
\hfill
\subfigure[Ground-truth]{\includegraphics[width=0.23\textwidth]{Figures/Surfaces/F3-\i-GroundTruth-surface.jpg}}
\hfill
\subfigure[\texttt{MST}]{\includegraphics[width=0.23\textwidth]{Figures/Surfaces/F3-\i-MST-surface.jpg}}
\hfill
\subfigure[\texttt{MRF}]{\includegraphics[width=0.23\textwidth]{Figures/Surfaces/F3-\i-MRF-surface.jpg}}
\hfill
\subfigure[\texttt{B-MRF}]{\includegraphics[width=0.23\textwidth]{Figures/Surfaces/F3-\i-maxW-1-surface.jpg}}
\caption{F3-Netherlands-\h}
\label{fig:F3-\h}
\end{figure*}
}

\foreach \h[count=\i] in {I,II} {
\begin{figure*}[h]
\hfill
\subfigure[Ground-truth]{\includegraphics[width=0.23\textwidth]{Figures/Surfaces/Opunaka-\i--GroundTruth-surface.jpg}}
\hfill
\subfigure[\texttt{MST}]{\includegraphics[width=0.23\textwidth]{Figures/Surfaces/Opunaka-\i--MST-surface.jpg}}
\hfill
\subfigure[\texttt{MRF}]{\includegraphics[width=0.23\textwidth]{Figures/Surfaces/Opunaka-\i--MRF-surface.jpg}}
\hfill
\subfigure[\texttt{B-MRF}]{\includegraphics[width=0.23\textwidth]{Figures/Surfaces/Opunaka-\i--maxW-1-surface.jpg}}
\hfill
\caption{Opunaka-3D-\h}
\label{fig:opunaka-\h}
\end{figure*}
}

\foreach \h[count=\i] in {I,II,III,IV} { 
\begin{figure*}[h]
\hfill
\subfigure[Ground-truth]{\includegraphics[width=0.23\textwidth]{Figures/Surfaces/R-WakaILXL-\i--GroundTruth-surface.jpg}}
\hfill
\subfigure[\texttt{MST}]{\includegraphics[width=0.23\textwidth]{Figures/Surfaces/R-WakaILXL-\i--MST-surface.jpg}}
\hfill
\subfigure[\texttt{MRF}]{\includegraphics[width=0.23\textwidth]{Figures/Surfaces/R-WakaILXL-\i--MRF-surface.jpg}}
\hfill
\subfigure[\texttt{B-MRF}]{\includegraphics[width=0.23\textwidth]{Figures/Surfaces/R-WakaILXL-\i--maxW-1-surface.jpg}}
\hfill
\caption{Waka-3D-\h}
\label{fig:waka-\h}
\end{figure*}
}

\end{document}